\pdfoutput=1
\RequirePackage[T1]{fontenc}
\documentclass[12pt]{article}
\usepackage[skip=1em plus 0.2em minus 0.1em]{parskip}
\usepackage[height=8.85in,width=6.45in]{geometry}

\usepackage[utf8]{inputenc}
\usepackage{amsmath}
\usepackage{amssymb}
\usepackage{mathtools}
\numberwithin{equation}{section}
\usepackage{slashed}
\usepackage{braket}
\usepackage{enumitem}
\usepackage[svgnames]{xcolor}
\usepackage[colorlinks,citecolor=DarkGreen,linkcolor=FireBrick,urlcolor=FireBrick,linktocpage,unicode]{hyperref}

\usepackage{tocloft}

\usepackage{footmisc}
\allowdisplaybreaks
\usepackage{graphicx}
\usepackage{tikz}
\usepackage{tikz-cd}
\usepackage{times}
 
\usepackage{bm}
\usepackage{physics}
\usepackage{xcolor}
\usepackage{natbib}
\usepackage{mdframed}
\usepackage{nicefrac}
\usepackage{booktabs}
\usepackage{lipsum}
\usepackage{titlesec}
\usepackage{wrapfig,lipsum,booktabs}
\usepackage{authblk}
\usepackage{blindtext}

\usepackage{algorithm}
\usepackage{algpseudocode}
\newcommand{\commentsymbol}{//}%
\algrenewcommand\algorithmiccomment[1]{\hfill {\footnotesize \commentsymbol{} #1}}

\usepackage{titletoc}
\usepackage{todonotes}
\usepackage{authblk}
\usepackage{setspace}
\usepackage{dsfont} %
\newcommand{\one}{\mathds{1}}

\usepackage[nameinlink,capitalize,noabbrev]{cleveref}
\crefdefaultlabelformat{#2#1#3}

\usepackage{CJK}

\usepackage{array}
\usepackage{bbm}
\usepackage{makecell}

\usepackage{dashbox}
\usepackage{xcolor}
\usepackage{colortbl}
\definecolor{lightyellow}{rgb}{1.0, 0.95, 0.7}
\definecolor{Blue}{rgb}{0, 0, 0.8}
\definecolor{blue}{rgb}{0,0,1}
\definecolor{darkgreen}{rgb}{0,0.40,0}
\definecolor{firebrick}{rgb}{0.698,0.133,0.133}

\definecolor{colorA}{rgb}{1,0,0}
\definecolor{colorB}{rgb}{0,0.3,1}
\definecolor{colorC}{rgb}{0.9,0.8,0.2}
\definecolor{colorD}{rgb}{0,0.65,0}
\definecolor{lesslightgray}{rgb}{0.5,0.5,0.5}
\definecolor{light-gray}{gray}{0.95}

\let\hat\widehat

\newcommand{\calF}{\mathcal{F}}

\newcommand{\calT}{\mathcal{T}}

\newcommand{\calX}{\mathcal{X}}

\newcommand{\diag}{\mathop{\rm{diag}}}

\newcommand{\Softmax}{{\rm{Softmax}}}
\newcommand{\relu}{{\rm{ReLU}}}
\newcommand{\attn}{{\rm{Attn}}}

\usepackage{pifont}
\newcommand{\vmark}{\ding{51}}
\newcommand{\xmark}{\ding{55}}

\def\R{\mathbb{R}}

\let\cite\citep 

\usepackage{amsthm}
\usepackage[many]{tcolorbox}
\usepackage{etoolbox}
\BeforeBeginEnvironment{proof}{\par\vspace{-1\baselineskip}}
\AfterEndEnvironment  {proof}{\par\vspace{-1.5\baselineskip}}

\newtheoremstyle{theoremstyle}
  {.5\baselineskip} %
  {.5\baselineskip} %
  {}                  %
  {}                  %
  {\bfseries}        %
  {.}                 %
  {1em}               %
  {}                  %

\theoremstyle{theoremstyle}
\newtheorem{theorem}{Theorem}[section]
\newtheorem{lemma}{Lemma}[section]

\newtheorem{proposition}{Proposition}[section]
\newtheorem{definition}{Definition}[section]

\newtheorem{remark}{Remark}[section]

\tcolorboxenvironment{theorem}{
  breakable,
  colback=black!10,
  colframe=white,%
  width=\dimexpr\linewidth+10pt\relax,%
  enlarge left by=-5pt,%
  enlarge right by=-5pt,%
  boxsep=5pt,%
  boxrule=0pt,
  left=0pt,right=0pt,top=0pt,bottom=0pt,
  sharp corners,
  before skip=0.5\baselineskip, %
  after skip=0.5\baselineskip,  %
  fonttitle=\bfseries, %
  coltitle=black %
}
\tcolorboxenvironment{remark}{
  blanker,
  breakable,
  before skip=.8\baselineskip,  %
  after  skip=.8\baselineskip   %
}

\tcolorboxenvironment{proposition}{
  breakable,
  colback=black!10,
  colframe=white,%
  width=\dimexpr\linewidth+10pt\relax,%
  enlarge left by=-5pt,%
  enlarge right by=-5pt,%
  boxsep=5pt,%
  boxrule=0pt,
  left=0pt,right=0pt,top=0pt,bottom=0pt,
  sharp corners,
  before skip=0.5\baselineskip, %
  after skip=0.5\baselineskip,  %
  fonttitle=\bfseries, %
  coltitle=black %
}

\tcolorboxenvironment{lemma}{
  breakable,
  colback=black!10,
  colframe=white,%
  width=\dimexpr\linewidth+10pt\relax,%
  enlarge left by=-5pt,%
  enlarge right by=-5pt,%
  boxsep=5pt,%
  boxrule=0pt,
  left=0pt,right=0pt,top=0pt,bottom=0pt,
  sharp corners,
  before skip=0.5\baselineskip, %
  after skip=0.5\baselineskip,  %
  fonttitle=\bfseries, %
  coltitle=black %
}

\tcolorboxenvironment{corollary}{
  breakable,
  colback=black!10,
  colframe=white,%
  width=\dimexpr\linewidth+10pt\relax,%
  enlarge left by=-5pt,%
  enlarge right by=-5pt,%
  boxsep=5pt,%
  boxrule=0pt,
  left=0pt,right=0pt,top=0pt,bottom=0pt,
  sharp corners,
  before skip=0.5\baselineskip, %
  after skip=0.5\baselineskip,  %
  fonttitle=\bfseries, %
  coltitle=black %
}

\tcolorboxenvironment{definition}{
  breakable,
  colback=black!10,
  colframe=white,%
  width=\dimexpr\linewidth+10pt\relax,%
  enlarge left by=-5pt,%
  enlarge right by=-5pt,%
  boxsep=5pt,%
  boxrule=0pt,
  left=0pt,right=0pt,top=0pt,bottom=0pt,
  sharp corners,
  before skip=0.5\baselineskip, %
  after skip=0.5\baselineskip,  %
  fonttitle=\bfseries, %
  coltitle=black %
}
\tcolorboxenvironment{assumption}{
  breakable,
  colback=black!10,
  colframe=white,%
  width=\dimexpr\linewidth+10pt\relax,%
  enlarge left by=-5pt,%
  enlarge right by=-5pt,%
  boxsep=5pt,%
  boxrule=0pt,
  left=0pt,right=0pt,top=0pt,bottom=0pt,
  sharp corners,
  before skip=0.5\baselineskip, %
  after skip=0.5\baselineskip,  %
  fonttitle=\bfseries, %
  coltitle=black %
}

\crefname{theorem}{Theorem}{Theorems}
\crefname{proposition}{Proposition}{Propositions}
\crefname{lemma}{Lemma}{Lemmas}
\crefname{corollary}{Corollary}{Corollaries}
\crefname{definition}{Definition}{Definitions}
\crefname{assumption}{Assumption}{Assumptions}
\crefname{remark}{Remark}{Remarks}
\crefname{problem}{Problem}{Problems}
\crefname{property}{Property}{property}

\tcolorboxenvironment{hypothesis}{
  breakable,
  colback=black!10,
  colframe=white,%
  width=\dimexpr\linewidth+10pt\relax,%
  enlarge left by=-5pt,%
  enlarge right by=-5pt,%
  boxsep=5pt,%
  boxrule=0pt,
  left=0pt,right=0pt,top=0pt,bottom=0pt,
  sharp corners,
  before skip=0.5\baselineskip, %
  after skip=0.5\baselineskip,  %
  fonttitle=\bfseries, %
  coltitle=black %
}
\crefname{hypothesis}{Hypothesis}{Hypothesises}

\tcolorboxenvironment{fact}{
  breakable,
  colback=black!10,
  colframe=white,%
  width=\dimexpr\linewidth+10pt\relax,%
  enlarge left by=-5pt,%
  enlarge right by=-5pt,%
  boxsep=5pt,%
  boxrule=0pt,
  left=0pt,right=0pt,top=0pt,bottom=0pt,
  sharp corners,
  before skip=0.5\baselineskip, %
  after skip=0.5\baselineskip,  %
  fonttitle=\bfseries, %
  coltitle=black %
}
\crefname{fact}{Fact}{Facts}

\tcolorboxenvironment{example}{
  breakable,
  colback=black!10,
  colframe=white,%
  width=\dimexpr\linewidth+10pt\relax,%
  enlarge left by=-5pt,%
  enlarge right by=-5pt,%
  boxsep=5pt,%
  boxrule=0pt,
  left=0pt,right=0pt,top=0pt,bottom=0pt,
  sharp corners,
  before skip=0.5\baselineskip, %
  after skip=0.5\baselineskip,  %
  fonttitle=\bfseries, %
  coltitle=black %
}
\crefname{example}{Example}{Examples}

\tcolorboxenvironment{question}{
  breakable,
  colback=black!10,
  colframe=white,%
  width=\dimexpr\linewidth+10pt\relax,%
  enlarge left by=-5pt,%
  enlarge right by=-5pt,%
  boxsep=5pt,%
  boxrule=0pt,
  left=0pt,right=0pt,top=0pt,bottom=0pt,
  sharp corners,
  before skip=0.5\baselineskip, %
  after skip=0.5\baselineskip,  %
  fonttitle=\bfseries, %
  coltitle=black %
}

\crefname{question}{Question}{Questions}
\numberwithin{equation}{section}
\numberwithin{theorem}{section}
\numberwithin{proposition}{section}
\numberwithin{definition}{section}
\numberwithin{lemma}{section}
\numberwithin{assumption}{section}
\numberwithin{remark}{section}

\usepackage{lipsum}

\newcommand*{\annot}[1]{\tag*{\footnotesize{\textcolor{black!50}{\big(#1\big)}}}}

\makeatletter
\let\save@mathaccent\mathaccent
\newcommand*\if@single[3]{%
    \setbox0\hbox{${\mathaccent"0362{#1}}^H$}%
    \setbox2\hbox{${\mathaccent"0362{\kern0pt#1}}^H$}%
    \ifdim\ht0=\ht2 #3\else #2\fi
}
\newcommand*\rel@kern[1]{\kern#1\dimexpr\macc@kerna}
\newcommand*\widebar[1]{\@ifnextchar^{{\wide@bar{#1}{0}}}{\wide@bar{#1}{1}}}
\newcommand*\wide@bar[2]{\if@single{#1}{\wide@bar@{#1}{#2}{1}}{\wide@bar@{#1}{#2}{2}}}
\newcommand*\wide@bar@[3]{%
    \begingroup
    \def\mathaccent##1##2{%
        \let\mathaccent\save@mathaccent
        \if#32 \let\macc@nucleus\first@char \fi
        \setbox\z@\hbox{$\macc@style{\macc@nucleus}_{}$}%
        \setbox\tw@\hbox{$\macc@style{\macc@nucleus}{}_{}$}%
        \dimen@\wd\tw@
        \advance\dimen@-\wd\z@
        \divide\dimen@ 3
        \@tempdima\wd\tw@
        \advance\@tempdima-\scriptspace
        \divide\@tempdima 10
        \advance\dimen@-\@tempdima
        \ifdim\dimen@>\z@ \dimen@0pt\fi
        \rel@kern{0.6}\kern-\dimen@
        \if#31
        \overline{\rel@kern{-0.6}\kern\dimen@\macc@nucleus\rel@kern{0.4}\kern\dimen@}%
        \advance\dimen@0.4\dimexpr\macc@kerna
        \let\final@kern#2%
        \ifdim\dimen@<\z@ \let\final@kern1\fi
        \if\final@kern1 \kern-\dimen@\fi
        \else
        \overline{\rel@kern{-0.6}\kern\dimen@#1}%
        \fi
    }%
    \macc@depth\@ne
    \let\math@bgroup\@empty \let\math@egroup\macc@set@skewchar
    \mathsurround\z@ \frozen@everymath{\mathgroup\macc@group\relax}%
    \macc@set@skewchar\relax
    \let\mathaccentV\macc@nested@a
    \if#31
    \macc@nested@a\relax111{#1}%
    \else
    \def\gobble@till@marker##1\endmarker{}%
    \futurelet\first@char\gobble@till@marker#1\endmarker
    \ifcat\noexpand\first@char A\else
    \def\first@char{}%
    \fi
    \macc@nested@a\relax111{\first@char}%
    \fi
    \endgroup
    }
\makeatother
\let\bar\widebar

\newcommand*{\email}[1]{\footnote{\href{mailto:#1}{\texttt{#1}}}}

\setlist[itemize,enumerate]{
  parsep=\parskip,                                   %
  itemsep=\dimexpr .3em - \parskip\relax plus 2pt,   %
  topsep=\dimexpr 6pt - \parskip\relax plus 1pt minus 1pt,
  partopsep=0pt,
  listparindent=\parindent
}

\begin{document}
\begin{titlepage}

\begin{flushright}
Last Update: \today
\end{flushright}

\vskip 2.5em
\begin{center}

{
\LARGE \bfseries %
\begin{spacing}{1.15} %
Transformer Approximations from ReLUs
\end{spacing}
}

\vskip 1em
Jerry Yao-Chieh Hu$^{\dagger*}$\email{jhu@u.northwestern.edu}
\quad
Mingcheng Lu$^{*}$\email{2860215400pp@gmail.com}
\quad
Yi-Chen Lee$^{\ddag}$\email{b10202055@ntu.edu.tw}
\quad
Han Liu$^{\dagger\sharp}$\email{hanliu@northwestern.edu}

\def\thefootnote{*}
\footnotetext{These authors contributed equally to this work.
}

\vskip 1em

{\small
\begin{tabular}{ll}
 $^\dagger\;$Center for Foundation Models and Generative AI, Northwestern University, Evanston, IL 60208, USA\\
 \hphantom{$^\ddag\;$}Department of Computer Science, Northwestern University, Evanston, IL 60208, USA\\
 $^\ddag\;$Department of Physics, National Taiwan University, Taipei 10617, Taiwan\\
 $^\sharp\;$Department of Statistics and Data Science, Northwestern University, Evanston, IL 60208, USA
\end{tabular}}

\end{center}

\noindent
We provide a systematic recipe for translating ReLU approximation results to softmax attention mechanism.
This recipe covers many common approximation targets. 
Importantly, it yields target-specific, economic resource bounds beyond universal approximation statements.
We showcase the recipe on multiplication, reciprocal computation, and min/max primitives. 
These results provide new analytical tools for analyzing softmax transformer models.

\end{titlepage}

{
\setlength{\parskip}{0em}
\setcounter{tocdepth}{2}
\tableofcontents
}
\setcounter{footnote}{0}

\section{Introduction}
\label{sec:intro}
We present a systematic recipe for translating ReLU approximation results to softmax Transformers\footnote{In this work, by Transformers, we mean attention-only Transformers implemented by the softmax attention mechanism \cite{vaswani2017attention}.}. 
Given a constructive ReLU approximator for a target, we construct an explicit softmax transformer with the same accuracy. 
The recipe applies to many common approximation targets and yields quantitative resource bounds beyond universal approximation statements.

This matters because broad \underline{U}niversal \underline{A}pproximation \underline{P}roperties (UAP) still dominate Transformer approximation theory.
For softmax Transformer, many universality results provide explicit constructions and quantitative resource bounds (e.g., parameters, depth, width...etc) \citep{yun2020universal,kajitsuka2023transformers,takakura2023approximation,jiang2024approximation,hu2025universal,liu2025attention}.
But these bounds reflect the cost of approximating broad function classes, not the complexity of a specific target.
As a result, they are often much looser than what practical scenarios require.
In many applications, the target class is much narrower.
So universality-oriented constructions contain high redundancy and provide weak worst-case complexity control (e.g., in  statistical estimation analyses).
In fact, this is one main reason existing worst-case bounds for practical transformer generative models remain loose \cite{hu2024statistical,su2025theoretical}.
What is missing is a constructive and target-specific approximation theory for softmax attention.

On the other hand, 
ReLU networks already admit a mature target-specific approximation theory.
For many concrete targets, one constructs ReLU networks with explicit depth and width as functions of approximation precision $\epsilon$ \citep{yarotsky2017error, telgarsky2017neural,petersen2018optimal,suzuki2018adaptivity,schmidt2020nonparametric,nakada2020adaptive, oko2023diffusion, fu2024unveil}.
Classic examples include multiplication, monomials, and reciprocal maps \citep{schmidt2020nonparametric, oko2023diffusion, fu2024unveil}.
These constructions often build a small set of primitives and then compose them, so rates propagate through the pipeline.
These nice results motivate us to seek the same style of results for Transformers, with explicit rates and constructions.
This is timely and of high practical importance since transformers are the core backbone in modern foundation models (LLMs and generative AIs).
Therefore, our goal in this paper is to transfer this target-specific ReLU theory to Transformers.

To this end, we present
a translation theorem to lift these ReLU constructions into softmax-attention Transformers.
Given a ReLU approximator, the theorem produces an explicit softmax-attention approximator with the same accuracy and controlled resources. 
To showcase this theorem, we illustrate the theorem on several basic targets: multiplication, reciprocal computation, min/max and clipping
in \cref{lem:trans_monomial,lem:trans_recipro,lem:min_max_approx}.

\paragraph{Contributions.}
Our contributions are three-fold:
\begin{itemize}
    \item \textbf{A General Translation Theorem (\cref{thm:maintex_approx_relu_trans}).}
    We prove a constructive reduction from ReLU approximation to softmax-attention Transformer approximation. 
    Given a ReLU approximator, the reduction produces an explicit Transformer approximator with the same target accuracy.
    Moreover, the obtained resource bounds are more economic than those of universality oriented.
    
    \item \textbf{Constructive Transformer Universality (\cref{thm:maintex_attention_uap,tab:uap_comparison}).}
    We establish a constructive transformer-native universal approximation result. 
    More precisely, our comparison point is \cite[Theorem G.1 in Appendix G]{hu2025universal}: their transformer-native attention-only universality theorem is obtained via an existence-type ReLU UAP. 
    In contrast, our result makes this route explicit and tracks the resulting architectural and norm resources.\footnote{More precisely, \citet{hu2025universal} contain two different universality results. 
    Their main-text result (\cite[Section 3 \& 4]{hu2025universal}) is constructive but uses an attached sequence-wise linear map, and therefore is not transformer-native in the attention-only sense considered here. 
    Their Theorem G.1 in Appendix G gives a transformer-native attention-only universality theorem, but it is derived via an existence-type ReLU UAP. 
    Our \cref{thm:maintex_attention_uap} makes this latter route constructive and tracks $H$, $W$, $K$, $\lambda$, $C_{KQ}$, and $C_V$.}

    \item \textbf{Economic Transformer Constructions for Specific Targets (\cref{lem:trans_monomial,lem:trans_recipro,lem:min_max_approx}).}
    Using our translation framework, we derive new economic transformer constructions for several specific approximation targets.
    In particular, we obtain economic softmax-attention Transformers to approximate rational functions, to simulate composite ReLU modules, and to implement schedule terms arising in generative models. 
    These case studies illustrate how our recipe yields  improved size and depth bounds for transformers on important classes of functions.
    They mirror the strong approximation results long enjoyed by ReLU networks.
\end{itemize}

\paragraph{Organization.}
\cref{sec:related_work} discusses related work.
\cref{sec:preliminary} presents preliminaries.
\cref{sec:main_theory} presents our main result: the ReLU-Transformer approximation translation theorem.
\cref{sec:implications} discusses key implications: 
a new universal approximation theory (\cref{subsec:transformer_universal_approximation}), 
a parsimonious set of primitives for rational function approximations (\cref{subsec:transformer_rational_approximation}), all by attention-only Transformers.
\cref{sec:5proofsketch} highlights our proof ideas.
\cref{sec:conclusion} presents concluding remarks.

\section{Related Work}
\label{sec:related_work}

\paragraph{Universal Approximation.}
For feedforward networks, \citet{pinkus1999approximation} gives classical universal approximation results with ReLU activation. Attention-based universality results then emerge for Transformer architectures. \citet{yun2020universal} prove universality for continuous sequence to sequence maps on compact domains via a contextual mapping construction with positional encodings. 
\citet{kajitsuka2023transformers} further simplify this route in a one-layer single-head setting.
\citet{takakura2023approximation} show that a one-layer transformer, with attention plus a token-wise FFN and a single embedding layer with positional encoding, approximates shift-equivariant $\alpha$-smooth functions.
\citet{jiang2024approximation} use the Kolmogorov representation theorem to obtain Jackson-type approximation rates for single-layer single-head transformers under explicit smoothness assumptions. 
Despite these advances, many universality proofs for Transformers still rely on post-attention FFNs to realize token-wise nonlinear transformations. 
\citet{hu2025universal} contain two different universality results. 
Their main-text theorem is constructive, but it uses an attached sequence-wise linear map and is therefore not transformer-native in our sense.
Their Theorem G.1 is transformer-native and attention-only, but it is derived from an existence-type ReLU UAP. 
In contrast, our \cref{thm:maintex_attention_uap} below takes \cite[Theorem G.1]{hu2025universal} as the direct comparison point and makes that route constructive with explicit resource tracking.

\paragraph{Target-Specific Approximation with ReLUs.}
Target-specific ReLU approximation goes beyond universal approximation by providing explicit constructions with a much lower network complexity.
Specifically, \citet{yarotsky2017error,petersen2018optimal,schmidt2020nonparametric} show ReLU networks can approximate monomials with additive error $\epsilon$ using only $O(1/\epsilon)$ nonzero parameters. \citet{telgarsky2017neural,boulle2020rational,oko2023diffusion} analyze low complexity ReLU constructions for approximating the reciprocal functions.
These primitives form a reusable dictionary that supports statistical guarantees and near minimax rates in applications including nonparametric regression \citep{suzuki2018adaptivity,schmidt2020nonparametric}, diffusion model score estimation \citep{oko2023diffusion,fu2024unveil}, and flow matching velocity estimation \citep{fukumizu2024flow}.

\paragraph{Target-Specific Approximation with Attentions.} 
Softmax attention implements target operators through token selection and value aggregation, and a score gap controls the selection error.
\citet{edelman2022inductive} relate the softmax selection error to the attention score gap, and this relation supports quantitative control for lookup type primitives. The one-hot positional codes require a width at least $N$, whereas almost orthogonal embeddings \citep{bhattamishra2024separations,sanford2023representational,sanford2024transformers} reduce this to $O(\log N)$ while maintaining index separability. \citet{sanford2023representational} use this device to approximate sparse averaging, where an attention head estimates the mean over a selected subset of tokens. \citet{bhattamishra2024separations} apply it for index-based retrieval primitives in string equality and nearest neighbor tasks. \citet{sanford2024transformers} incorporate related routing primitives into larger algorithmic constructions, using gap-controlled softmax attention to route information across positions. \citet{yang2025transformer} systematize these lookup and averaging modules and summarize the associated gap-based error control heuristics. In contrast, we give target-specific approximation guarantees for continuous primitives, translating low complexity ReLU approximations into attention-based modules with explicit dependence on the approximation error in both architectural complexity and norm control.

\section{Preliminaries}
\label{sec:preliminary}
We introduce the two model classes we build on: (i) ReLU networks, and (ii) softmax-attention Transformer networks with explicit norm constraints.

\paragraph{Notations.}
For a positive integer $n$, let $[n]:=\{1,\dots,n\}$. 
Denote a length-$n$ input sequence with token dimension $d$ by $X:=(x_1,\dots,x_n)\in\R^{d\times n}$, where $x_j$ is the embedding of the $j$-th token. 
For a vector $a$, let $a_j$ denote its $j$-th entry. 
For a matrix $X$, let $X_{i,j}$ denote its $(i,j)$-th entry. 
Let $e_k\in\R^n$ denote the $k$-th standard basis (i.e., one-hot) vector.
For $x\in\R^n$, let $\|x\|_2$ and $\|x\|_\infty$ denote the $\ell_2$- and $\ell_\infty$-norms. 
For a matrix $A\in\R^{d\times n}$,  $\|A\|_F$ denote the Frobenius norm,  $\|A\|_{p,q}:=\sup_{v\neq 0}\frac{\|Av\|_q}{\|v\|_p}$ denote the induced operator norm for $1\le p,q\le\infty$, and  $\|A\|_{\rm max}:=\max_{i\in[d],j\in[n]} |A_{i,j}|$ denote the entrywise max norm. 
Finally, for $a,b\in\R$, let $a\lor b:=\max\{a,b\}$.

\paragraph{ReLU Network.}
We first introduce the standard ReLU network function class.

\begin{definition}
[$K$-Layer ReLU Network Class]
\label{def:ffn_class}
Let $\relu(t):=\max\{t,0\}$ denote the coordinate-wise ReLU activation for $t\in\R$.
Fix $K,W,S\in\mathbb{N}$, $B>0$, input dimension $d_1\in\mathbb{N}$, and output dimension $d_{K+1}\in\mathbb{N}$. 
Define $\mathcal{F}(K,W,S,B)$ as the class of functions
\begin{align*}
& \calF(K, W, S, B) \coloneqq \Big\{f:\mathbb{R}^{d_1}\to\mathbb{R}^{d_{K+1}}\Big| (A_K \mathrm{ReLU}[\cdot] + b_K) \circ \cdots
\circ (A_2 \mathrm{ReLU}[\cdot] + b_2) \circ (A_1 x + b_1) \Big\},
\end{align*}
where $A_i\in\mathbb{R}^{d_{i+1}\times d_i}$,  $b_i\in\mathbb{R}^{d_{i+1}}$  for $i\in[K]$, and
\begin{align*}
\max_{i\in[K+1]} d_i \le W,\qquad
\max_{i\in[K]} \bigl\{\|A_i\|_\infty \lor \|b_i\|_\infty\bigr\}\le B,\qquad
\sum_{i=1}^K \bigl(\|A_i\|_0+\|b_i\|_0\bigr)\le S.
\end{align*}
Here $\|\cdot\|_0$ denotes the number of nonzero entries.
\end{definition}

\paragraph{Softmax Attention Network.}
Let's start with the softmax function on vectors and matrices.
\begin{definition}[Temperature-Scaled Softmax]\label{def:temperature_softmax}
For $\lambda>0$ and vector $u\in\R^d$, define
\begin{align*}
\Softmax_\lambda(u)_i := \frac{\exp(\lambda u_i)}{\sum_{j=1}^d \exp(\lambda u_j)},\quad i\in[d].  
\end{align*}
For matrix $X\in\R^{d\times n}$, define the column-wise softmax
\begin{align*}
\Softmax_\lambda(X)_{:,k} := \Softmax_\lambda(X_{:,k}),\quad k\in[n].
\end{align*}
\end{definition}

We next define multi-head self-attention layer. 
\begin{definition}
[Multi-Head Attention Layer]
\label{def:att_layer}
Fix the number of heads $H\in\mathbb{N}$, $\lambda>0$, and dimensions $d_{\mathrm{in}},d_{\mathrm{out}}\in\mathbb{N}$.
An $H$-head self-attention layer is a map $\attn:\R^{d_{\rm in}\times n}\to \R^{d_{\rm out}\times n}$ of the form
\begin{align*}
\mathrm{Attn}(Z) = \sum_{h = 1}^H W_V^{(h)} Z 
\mathrm{Softmax}_\lambda ( ( W_K^{(h)} Z )^\top W_Q^{(h)} Z ),
\end{align*}
where $W_K^{(h)}, W_Q^{(h)} \in \R^{d_h \times d_{\mathrm{in}}}$, and $W_V^{(h)}\in\mathbb{R}^{d_{\mathrm{out}}\times d_{\mathrm{in}}}$ for all $h\in[H]$.
We omit $\lambda$ when $\lambda=1$.
When convenient, we write $W_{KQ}^{(h)}=(W_K^{(h)})^\top W_Q^{(h)}\in\mathbb{R}^{d_{\mathrm{in}}\times d_{\mathrm{in}}}$.
\end{definition}

We now define the Transformer function class as compositions of multi-head attention layers\footnote{In this paper, by Transformers, we mean attention-only Transformer networks. 
Please see \cref{sec:conclusion} for discussions regrading practical Transformers (including FFNs and residual connections).}
\begin{definition}
[Transformer Network Class]
\label{def:trans_funciton_class}
Let $X \in \R^{d_1 \times n}$ be the input sequence.  
Fix $H,W,K\in\mathbb{N}$, $C_{KQ},C_V>0$.
Define the Transformer network class
\begin{align*}
\calT(H, W, K, C_{KQ}, C_V) 
\coloneqq
\Big\{  \tau:\R^{d_1 \times n} \rightarrow \R^{d_{K+1} \times n} 
 | 
\tau
= T\circ \mathrm{Attn}^{(K)} \circ \cdots  \circ \mathrm{Attn}^{(1)}\circ P \Big\},
\end{align*}
where $\mathrm{Attn}^{(j)}:\R^{d_{j}\times n}\to\R^{d_{j+1}\times n}$ is the $j$-th $H$-head attention layer for all $j \in [K]$, $P$ is a fixed preprocessing layer, 
$T$ is the truncation layer that deletes the last token and $W \coloneqq \max_{j, h} \{ d_h, d_j \}$. 
Further,
\begin{align*}
\max_{j, h} \| W_V^{(h,j)} \|_F := C_V, 
\quad \max_{j, h} \| W_{KQ}^{(h,j)}\|_F := C_{KQ}, 
\end{align*}
where $W_{KQ}^{(h,j)}$ and $W_V^{(h,j)}$ are the weight matrices in the $h$-th head of the $j$-th attention layer.
\end{definition}

\section{Main Result: Translation Theorem}
\label{sec:main_theory}

We aim to build  transformer approximation theory by reducing it to the mature approximation theory of ReLU networks.
The key point is to make this reduction constructive and quantitative. 
In particular, we track resources such as width, depth, and parameter bounds.
This is motivated by a gap in the literature. Existing transformer approximation results are mostly universality-oriented \citep{suzuki2018adaptivity,kajitsuka2023transformers,hu2025universal,cheng2025unified}. In contrast, ReLU networks admit much sharper target-specific approximation theory \citep{yarotsky2017error,petersen2018optimal,oko2023diffusion,fu2024unveil}.
As discussed in \cref{sec:intro}, 
establishing similarly economic target-specific resource bounds is very useful for analyzing transformer models.

To address this gap, we prove a constructive translation from ReLU approximation to softmax-attention approximation.
Given any ReLU approximator on a compact domain, we construct a softmax-attention network that approximates the same target. 
We also obtain quantitative bounds on its architectural and norm resources. 
Thus, existing ReLU approximation results lead to attention approximation results with controlled overhead.
We proceed in two steps. 
First, \cref{lem:maintex_approximate_one_layer_relu_with_attention_matrix} gives a fixed-depth translation for a single ReLU layer. 
Namely, we construct an attention network that approximates a single ReLU layer to arbitrary precision. 
Then \cref{thm:maintex_approx_relu_trans} lifts this construction to general-depth ReLU networks by composition. 
As a result, we obtain the full ReLU-to-softmax-attention translation theorem with explicit approximation and parameter bounds.

We begin with the case of approximating one-layer ReLU with Transformer.

\begin{lemma}[Constructive Approximation of One-Layer ReLU Networks by Transformer]
\label{lem:maintex_approximate_one_layer_relu_with_attention_matrix}
Let $\mathcal{X}\subset\R^{d\times n}$ be compact. For $X=(x_1,\dots,x_n)\in\mathcal{X}$, let $f:\mathcal{X}\to\R^{d\times n}$ be defined entry-wise by
\begin{align*}
[f(X)]_{r,i}
=
\sum_{k=1}^N a_{i,k}^{(r)}
\relu (\sum_{j=1}^n (w_{i,k,j}^{(r)})^\top x_j),
\qquad r\in[d],\ i\in[n],
\end{align*}
where $a_{i,k}^{(r)}\in\{-1,1\}$ and $w_{i,k,j}^{(r)}\in\R^d$.
Then, for any $0<\epsilon<1$, there exists a Transformer
\begin{align*}
\phi = T\circ \mathrm{Attn}^{(3)}\circ \mathrm{Attn}^{(2)}\circ \mathrm{Attn}^{(1)}\circ P,
\end{align*}
such that
\begin{align*}
\|\phi(X)-f(X)\|_{\max}\le\epsilon
\qquad\text{for all } X\in\mathcal{X},
\end{align*}
where $P$ is the preprocessing map that pads one zero token, and $T$ is the truncation map that deletes the last token.
Further, suppose $\calX\in[-C_X, C_X]^{d\times n}$ and $\|w_{i,k,j}^{(r)}\|_\infty\leq w_0$ for all $i,j \in [n], k \in [N], r\in [d]$.
For $C_X\geq1$, parameter bounds of constructed Transformer $\phi$ satisfy
\begin{align*}
  & ~H = O(N),\quad W=O(N), \quad
  \lambda = O(\frac{N\ln(NC_Xw_0/\epsilon)}{\epsilon}),\\ 
  & ~ \|W_V\|_F = O(w_0\sqrt{N}), \quad \|W_{KQ}\|_F=O(\sqrt{\ln\frac{N C_X w_0}{\epsilon}}),
\end{align*}
where $O(\cdot)$ hides polynomial factors depending on $d$ and $n$.
\end{lemma}

\begin{proof}
Please see \cref{subsec:main_proof_one_layer_relu} for a detailed proof.
\end{proof}

Importantly, \cref{lem:maintex_approximate_one_layer_relu_with_attention_matrix} does not only prove existence. 
It quantifies the cost of approximating a one-layer ReLU network by a three-block Transformer network. 
The head count $H$ measures the parallel routing needed to realize the $N$ hidden units across the $d$ output coordinates. 
The inverse temperature $\lambda$ controls how sharply softmax concentrates on the maximal score, and hence how well it approximates hardmax-type routing. 
Explicit norm bounds on $W_V$ and $W_{KQ}$ are also essential. 
They are needed in the next step, where we compose these modules across layers and control the growth of errors and magnitudes.

We then lift the three-block construction in \cref{lem:maintex_approximate_one_layer_relu_with_attention_matrix} to arbitrary depth by composition. 
For each ReLU layer, we construct a constant-depth attention-only Transformer to approximate that layer on the relevant compact domain. 
Stacking these modules gives a Transformer network whose depth scales linearly with $K_f$.

\begin{theorem}[Approximating Multi-Layer ReLU Networks by Transformer]
\label{thm:maintex_approx_relu_trans}
Let $X:=(x_1,\dots,x_n)\in [-C_X, C_X]^{d \times n}$ and  ${\rm vec}(X):=(x_1^\top,\dots,x_n^\top)\in[-C_X, C_X]^{dn}$.
For any $\epsilon \in (0,1)$ and any ReLU network
$f \in \mathcal{F}(K_f, W_f, S, B)$, there exists  a Transformer
$g \in \mathcal{T}(H, W, K, C_{KQ}, C_V)$
such that
\begin{align*}
\| g(X) - f({\rm vec}(X)) \|_{\max} \leq \epsilon
\end{align*}
for all $X \in [-C_X, C_X]^{d \times n}$,
where $g$ takes the form
$g = T \circ {\rm Attn}_{3K_f} \circ \cdots \circ {\rm Attn}_1 \circ P$
with $P$ the preprocessing layer and $T$ the truncation layer.
For $C_X\geq1$, the parameters of $g$ satisfy
\begin{align*}
    & ~H = O( W_f),
\quad
W = O(W_f),
\quad
K = O(K_f),
\quad \quad C_V = O(B\sqrt{W_f}), \\
& ~ C_{KQ}
= O(\sqrt{K_f \ln(\frac{K_f W_f B C_X}{\epsilon})}), \quad
\lambda= O(\frac{K_f^2W_f^{K_f+1} B^{K_f} \ln(K_f W_f B C_X / \epsilon)}{\epsilon}).
\end{align*}
\end{theorem}

\begin{proof}
Please see \cref{sec:5proofsketch} for a proof sketch and \cref{subsec:main_proof_multi_layer_relu} for a detailed proof.
\end{proof}

Here we highlight three key implications of \cref{thm:maintex_approx_relu_trans}:
\begin{itemize}
    \item \textbf{General Translator.}
    \cref{thm:maintex_approx_relu_trans} is a compiler from the approximation theory of ReLU networks to that of softmax-attention transformers.
    Existing transformer universality results often target broad function classes, leading to large resource bounds reflecting the cost of universality rather than the complexity of a specific target \citep{kajitsuka2023transformers,hu2025universal}. 
    By contrast, decades of ReLU approximation theory provide tight, target-adaptive constructions for concrete primitives and structured function families \citep{yarotsky2017error,petersen2018optimal,oko2023diffusion,fu2024unveil}. 
    Our translation theorem leverages this maturity: once a ReLU approximator with a quantitative guarantee is available, \cref{thm:maintex_approx_relu_trans} converts it into a transformer approximator with controlled overhead, thereby inheriting sharper rates and more informative complexity bounds.

    \item \textbf{Constructive.}
    A second key feature is that the reduction is constructive and comes with controlled resources.
    Beyond asserting existence, our proof yields an explicit attention architecture and quantitative bounds on the depth, width, number of heads, inverse temperature, and weight norms.
    Such resource control is indispensable for a quantitative approximation theory, as it makes the reduction reproducible, comparable across model classes, and suitable for downstream complexity estimates.
    Notably, the depth and width scale as $K = O(K_f)$ and $W = O(W_f)$, so the translation preserves the architectural complexity of the source ReLU network up to constant factors.

    \item \textbf{Useful Analytic Tool.} 
    Finally, \cref{thm:maintex_approx_relu_trans} has a broad scope.
    It applies to any ReLU network in $\calF(K_f,W_f,S,B)$ and therefore  extends to any function class with an available ReLU approximation.
    Any improved ReLU approximation theorem, whether it concerns universal approximation or approximation of specific function classes, yields a corresponding transformer approximation theorem by composition with our translation.
    In \cref{sec:implications}, we demonstrate how to use the translation theorem as a plug-in tool. 
    First, we derive a transformer universal approximation result with explicit bounds. 
    Second, we translate known ReLU constructions to obtain quantitative approximation guarantees for selected simple target functions.
\end{itemize}

\section{Implications of Translation Theorem}
\label{sec:implications}

In this section, we demonstrate \cref{thm:maintex_approx_relu_trans} as a plug-in tool. 
It converts ReLU approximation results into softmax-attention Transformer approximation results. 
This gives two types of consequences. 
First, classical ReLU universal approximation theorems on compact domains lift to softmax-attention Transformers. 
Second, constructive ReLU designs for standard targets yield constructive Transformer approximators with explicit constructions and resource bounds.

We illustrate these consequences in several directions. 
First, \cref{subsec:transformer_universal_approximation} derives a universal approximation theorem for softmax attention from a classical ReLU universal approximation result. 
Next, \cref{subsec:transformer_rational_approximation} develops a modular framework for rational approximation based on monomial, reciprocal, and clipping/min-max primitives. 
Finally, \cref{sec:outline} presents additional target-specific attention approximations that arise in practice.

\subsection{Constructive Transformer-Native Universal Approximation}
\label{subsec:transformer_universal_approximation}

We establish a universal approximation theorem for softmax attention with explicit parameter bounds as functions of the approximation error.

\begin{theorem}[Constructive Transformer Universal Approximation]\label{thm:maintex_attention_uap}
Let $n,r\in\mathbb{N}$ and let  $\mathcal{F}_{n,r}$ denote  the unit ball of Sobolev space $W^{r,\infty}([0,1]^n)$ from \cref{def:sobolev_space}.
Then for any $f^\star \in \mathcal{F}_{n,r}$ and  any $\epsilon \in (0,1)$, there exists a Transformer network
$g \in \calT(H,W,K,C_{KQ},C_V)$
such that
\begin{align*}
|g(x) - f^\star(x)|
\leq
\epsilon
\end{align*}
for all $x \in [0,1]^n$.
The network parameters satisfy
\begin{align*}
& ~H = O(\delta^{n/r}\ln\delta),
\quad
W = O(\delta^{n/r}\ln \delta),
\quad
K = O(\ln \delta),
\\
& ~  C_V = O(\delta^{(n+2)/2r}\sqrt{\ln\delta}), \quad  C_{KQ}
= O(\ln\delta), \quad
\lambda
= O(\delta^{\mathrm{polylog}(\delta)}\cdot\ln^{\mathrm{polylog}(\delta)}\delta ),
\end{align*}
where $\delta=\epsilon^{-1}$ and $O(\cdot)$ hides polynomial factors in $n,r$.
\end{theorem}

\begin{table*}[t]
\centering
\caption{\small
\textbf{Comparison of Universal Approximation Results Most Relevant to \cref{thm:maintex_attention_uap}.}
Here ``in our class'' means belonging to the attention-only softmax Transformer class in \cref{def:att_layer}, without post-attention FFNs or other not Transformer-native parts.
Equivalently, the approximation capability comes solely from the softmax attention mechanism.
Here ``constructive'' means that the approximating architecture is constructed end-to-end.
The direct benchmark for \cref{thm:maintex_attention_uap} is \cite[Theorem G.1]{hu2025universal}.
}
\label{tab:uap_comparison}
\vspace{.5em}
\renewcommand{\arraystretch}{1.18}
\resizebox{\textwidth}{!}{%
\begin{tabular}{
>{\raggedright\arraybackslash}m{0.20\textwidth}|
>{\raggedright\arraybackslash}m{0.27\textwidth}|
>{\centering\arraybackslash}m{0.09\textwidth}|
>{\centering\arraybackslash}m{0.10\textwidth}|
>{\raggedright\arraybackslash}m{0.32\textwidth}
}
\hline
\multicolumn{1}{@{}c|}{\textbf{Result}}
&
\multicolumn{1}{c|}{\textbf{Model class}}
&
\multicolumn{1}{c|}{\textbf{In our class?}}
&
\multicolumn{1}{c|}{\textbf{Constructive?}}
&
\multicolumn{1}{c@{}}{\textbf{Resource statement}}
\\
\hline
\citep{yun2020universal}
&
Uses self-attention, FFNs, and positional encodings.
&
\xmark
&
\vmark
&
Gives explicit quantitative UAP, but lies outside our class due to FFNs.
\\
\hline
\citep{kajitsuka2023transformers};
\citep[Appendix B]{hu2024fundamental}
&
Uses one-layer one-head self-attention, together with FFNs for UAP.
&
\xmark
&
\xmark
&
Gives an existential UAP route via contextual mapping with low-rank attention weights; \citet{hu2024fundamental} extend to generic weights, but both the UAP still rely on FFNs and are nonconstructive.
\\
\hline
\cite[Theorem~3.3]{hu2025universal}
&
Uses multi-head softmax attention together with attached sequence-wise linear maps.
&
\xmark
&
\vmark
&
Gives a constructive attention-based UAP, but lies outside our class due to the attached linear maps.
\\
\hline
\cite[Theorem G.1]{hu2025universal}
&
Uses a transformer-native attention-only softmax network.
&
\vmark
&
\xmark
&
Gives a transformer-native UAP, but obtains it by invoking an existence-type ReLU UAP.
\\
\hline
\cref{thm:maintex_attention_uap}
&
Uses a transformer-native attention-only softmax network.
&
\vmark
&
\vmark
&
Gives a constructive transformer-native UAP with explicit bounds on $H$, $W$, $K$, $\lambda$, $C_{KQ}$, and $C_V$.
\\
\hline
\end{tabular}%
}
\end{table*}

\begin{proof}
Please see \cref{subsec:transformer_uap} for a detailed proof.
\end{proof}

Intuitively, we obtain \cref{thm:maintex_attention_uap} by combining the classical ReLU universal approximation result of \cite{yarotsky2017error} with \cref{thm:maintex_approx_relu_trans}. 
The key distinction from prior Transformer universality results is that the constructiveness within the transformer-native attention-only route. 
Specifically, our comparison point is \cite[Theorem G.1 in Appendix G]{hu2025universal}, which establishes transformer-native attention-only universality via an existence-type ReLU UAP. 
Here, in contrast, we combine the explicit ReLU construction in \cite{yarotsky2017error} with \cref{thm:maintex_approx_relu_trans}. 
This gives a constructive Transformer network together with explicit dependence on the architectural and norm parameters. 
Please see \cref{tab:uap_comparison} for an overview.
Two more remarks are in order:

\begin{remark}
\cref{thm:maintex_attention_uap} extends to sequence-to-sequence maps and general input domain trivially.
Please see \cref{rem:extension_of_uap} for detailed discussion.
\end{remark}

\begin{remark}
    \cref{thm:maintex_attention_uap} provides a guarantee for attention-based approximation over a broad class of target functions. 
    In particular, every bounded and Lipschitz continuous function on compact domain belongs to $c\mathcal{F}_{n,1}$ for some $c>0$.
    Thus, the result already covers a wide range of practically relevant functions. 
    On the other hand, higher Sobolev regularity leads to sharper parameter bounds. 
    This shows that smoother functions admit more efficient Transformer approximation. 
\end{remark}

\subsection{Approximation of Rational Functions}
\label{subsec:transformer_rational_approximation}

This section gives target-specific approximation results beyond universality. 
We use \cref{thm:maintex_approx_relu_trans} to translate sharp ReLU constructions into Transformer networks with explicit resource bounds.

Such target-specific constructions may be useful in statistical theory of Transformer models.
In the ReLU literature, similar modules lead to adaptive or near-minimax guarantees in applications such as nonparametric regression \cite{suzuki2018adaptivity,schmidt2020nonparametric}, score-matching diffusion \cite{oko2023diffusion,fu2024unveil}, and flow matching model estimations \cite{fukumizu2024flow}. 
Many modern generative models now use transformer backbones, including diffusion transformers and related flow-matching variants. 
Thus, analogous Transformer constructions may provide useful ingredients for future statistical analyses of Transformer-based generative models.

We focus on rational function approximation. 
Many smooth targets reduce to polynomial evaluation and division on bounded domains. 
This reduction appears in Taylor expansions and in classical approximation schemes such as splines, wavelets, and kernel methods.
Accordingly, we derive Transformer approximations for three basic primitives by \cref{thm:maintex_approx_relu_trans}: 
\begin{itemize}
    \item Monomials, 
    \item Reciprocal function, 
    \item Entry-wise minimum and maximum with clipping,
\end{itemize}
through \cref{lem:trans_monomial,lem:trans_recipro,lem:min_max_approx}.
These primitives are composable. 
Together they form a small toolkit for approximating rational functions with Transformer. 
\cref{sec:outline} then gives further examples built from the same recipe.

\paragraph{Approximation of Monomials.}
Monomials are basic building blocks for polynomial and rational approximation.
A highlight of ReLU approximation theory is that it achieves an approximation error $\epsilon$ for monomials using only $O(\ln(\epsilon^{-1}))$ nonzero network parameters \citep{yarotsky2017error,petersen2018optimal,schmidt2020nonparametric}.
The next proposition translates this type of result to softmax-attention Transformers.

\begin{proposition}
[Approximation of Monomials]
\label{lem:trans_monomial}
Let $C_X \geq 1$ and $d, n \in \mathbb{N}$ with $dn > 2$. Then, for any $\epsilon \in (0,1)$, there exists $f_{\mathrm{mult}} \in \calT(H,W,K,C_{KQ}, C_V)$ such that
\begin{align*}
\abs{   f_{\mathrm{mult}}(X')  - \prod_{j=1}^n \prod_{i=1}^d X_{ij} } \leq \epsilon + dn C_X^{dn - 1}\epsilon_{ \mathrm{error} }
\end{align*}
for all  $X \in [-C_X, C_X]^{d \times n}$ and $X' \in \R^{d \times n}$ satisfying $\| X - X' \|_\infty \leq \epsilon_{ \mathrm{error} }$.
For $C_X\geq1$, the network parameters satisfy: 
\begin{align*}
        & ~H = O(1),
\quad
W = O(1),
\quad
K = O(\ln \delta),
\quad C_V = O(C_X^{dn}), \\
& ~
C_{KQ}
= O(\ln \delta), \quad 
\lambda = O(\delta\cdot\ln^3 \delta\cdot(dn)^{\mathrm{polylog}(\delta)} C_X^{\mathrm{polylog}(\delta)} ),
\end{align*}
where $\delta=C_X/\epsilon$, $\mathrm{polylog}(\delta)$ denotes a polynomial in the logarithm of $\delta$ and $O(\cdot)$ hides polynomial factors depending on $d$ and $n$.
\end{proposition}

\begin{proof}
Please see \cref{sec:proof_monomial} for a detailed proof.
\end{proof}

\paragraph{Approximation of Reciprocal Functions.}
Building on the monomial approximation above, we now turn to the reciprocal map. 
Monomial approximators let us approximate each term of a polynomial, and summing these terms gives a polynomial approximator. 
The reciprocal then enables division on domains bounded away from zero. 
Together, these ingredients form the basis of rational function approximation.
The next proposition constructs an explicit softmax-attention Transformer network for approximating $x\mapsto 1/x$ on a domain bounded away from zero. 
It also quantifies the stability of this approximation under input perturbations.
\begin{proposition}
[Approximation of the Reciprocal Function]
\label{lem:trans_recipro}
For any approximation error $ \epsilon_{} \in (0,1)$, there exists a Transformer network $ f_{ \mathrm{inv} } \in \calT(H, W, K, C_{KQ}, C_V) $ such that
\begin{align*}
\abs{ f_{ \mathrm{inv} }(x') - \frac{1}{x}  } \leq 2\epsilon + \frac{ \abs{x' - x} }{ \epsilon_{}^2 } ,
\end{align*}
for all $x \in [\epsilon, 1/\epsilon]$ and $x' \in [-C_X, C_X]$. 
For $C_X\geq1$, the network parameters satisfy:
\begin{align*}
& H = O(\ln^3\delta),
\quad
W = O(\ln^3 \delta),
\quad
K = O(\ln^2 \delta),
\quad C_V = O(\delta^2\sqrt{\ln^3\delta}), 
\\
&  C_{KQ}
= O(\sqrt{\ln^2\delta\cdot\ln (\delta C_X)}), 
\quad
\lambda
= O(\delta^{\mathrm{polylog}(\delta)}\cdot\ln^{\mathrm{polylog}(\delta)}\delta\cdot\ln C_X ),
\end{align*}
where $\delta=1/\epsilon$ and $\mathrm{polylog}(\delta)$ denotes a polynomial in the logarithm of $\delta$.
\end{proposition}

\begin{proof}
Please see \cref{sec:proof_recipro} for a detailed proof.
\end{proof}

\paragraph{Approximation of Min, Max Operation and Clipping.}

Rational approximation also requires controlling the input range. 
In many statistical settings, an unbounded input domain invalidates convergence guarantees. 
In particular, denominators must stay bounded away from zero for the reciprocal module to apply. 
We therefore introduce entry-wise minimum, maximum, and clipping.
These operations enforce the boundedness conditions needed in later constructions.

We first formalize the clipping map as follows.
\begin{definition}[Clipping Function]
For $c \in (0,C_X)$, we define the scalar clipping map
\begin{align*}
{\rm clip}_c(t) := \max\{-c,\min\{t,c\}\}, \qquad t \in \R.
\end{align*}
We extend it entry-wise to matrices by
\begin{align*}
{\rm clip}_c(X)_{i,j} := {\rm clip}_c(X_{i,j}), \qquad X \in \R^{d \times n}.
\end{align*}
\end{definition}

The next proposition approximates entry-wise minimum, maximum and clipping  operations with Transformers to arbitrary precision.
\begin{proposition}
[Entry-wise Minimum, Maximum and Clipping]
\label{lem:min_max_approx}
Let $0<c<C_X$.
For any $\epsilon\in(0,1)$, there exist  Transformer networks
\begin{align*}
    f_{\max}, f_{\min},f_{\rm clip}\in\mathcal{T}(H,W,K,C_{KQ},C_V)
\end{align*}
such that, for all $X,Y\in[-C_X,C_X]^{d\times n}$,
\begin{align*}
    \|f_{\max}(X,Y) - \max(X,Y)\|_{\max} \leq & ~\epsilon, \\
    \|f_{\min}(X,Y) - \min(X,Y)\|_{\max} \leq & ~\epsilon, \\
    \|f_{\rm clip}(X) - {\rm clip}_c(X)\|_{\max} \leq & ~\epsilon,
\end{align*}
where $\max(\cdot,\cdot)$ and $\min(\cdot,\cdot)$ denote entry-wise
operations.
For $C_X\geq1$, the network parameters satisfy
\begin{align*}
        & ~H = O(1),
\quad
W = O(1),
\quad
K = O(1),
\end{align*}
and, for $f_{\rm max}$ and $f_{\min}$:
\begin{align*}
    C_V = O(1), \quad C_{KQ}= O(\sqrt{\ln(\delta C_X)}), \quad \lambda =O(\delta\ln (\delta C_X) ),
\end{align*}
for $f_{\rm clip}$:
\begin{align*}
    C_V = O(C_X), \quad C_{KQ}= O(\sqrt{\ln(\delta C_X)}), \quad \lambda =O(\delta C_X\ln (\delta C_X) ),
\end{align*}
where $\delta=1/\epsilon$ and $O(\cdot)$ hides polynomial factors depending on $d$ and $n$.
\end{proposition}

\begin{proof}
Please see \cref{sec:proof_min_max} for a detailed proof.
\end{proof}

To this end, \cref{lem:trans_monomial,lem:trans_recipro,lem:min_max_approx} provide a modular toolbox for rational function approximation with Transformers. 
Monomial modules evaluate numerator and denominator polynomials. 
Clipping keeps the denominator in a valid range. 
The reciprocal module then implements division. 
Moving forward, \cref{sec:outline} develops further examples from the same toolkit.

\subsection{Additional Target-Specific Approximation by Transformers}
\label{sec:outline}
This section presents more examples of target-specific approximation with Transformers from the toolkit in \cref{subsec:transformer_rational_approximation}.
The point is to show that \cref{thm:maintex_approx_relu_trans} also applies to concrete operators beyond the basic rational primitives. We focus on two examples: the square-root map and several closed-form time schedules that arise in generative models, e.g., \cite{fu2024unveil,su2025theoretical}.

\paragraph{Square Root Operator.}
Besides the monomials and the reciprocal modules (\cref{lem:trans_monomial,lem:trans_recipro}),
rational function constructions often require the square-root map. 
This occurs when intermediate expressions involve positive rational quantities. 
The next proposition constructs an explicit Transformer  for approximating $x\mapsto \sqrt{x}$ on a domain bounded away from zero. 
It also quantifies how input perturbations contribute to the approximation error.

\begin{proposition}
[Approximation of the Square Root]
\label{prop:srq_rt}
Let $C_X > 0$.
Then, for any $\epsilon \in (0,1)$, there exists $f_{\mathrm{root}} \in \calT(H, W, K, C_{KQ}, C_V)$ such that
\begin{align*}
|f_{\mathrm{root}}(x') - \sqrt{x}| \leq 2\epsilon + \frac{|x' - x|}{\sqrt{\epsilon}}, \quad \text{for all } x' \in [-C_X, C_X], x \in [\epsilon, 1/\epsilon].
\end{align*}
For $C_X\geq1$, the network parameters satisfy:
\begin{align*}
            & ~H = O(\ln^3\delta),
\quad
W = O(\ln^3 \delta),
\quad
K = O(\ln^2 \delta),
\quad C_V = O(\delta\sqrt{\ln^3\delta}), 
\\
& ~ C_{KQ}
= O(\sqrt{\ln^2\delta\cdot\ln (\delta C_X)}), \quad
\lambda
= O(\delta^{\mathrm{polylog}(\delta)}\cdot\ln^{\mathrm{polylog}(\delta)}\delta\cdot\ln  C_X ).
\end{align*}
where $\delta = 1/\epsilon$ and $\mathrm{polylog}(\delta)$ denotes a polynomial in the logarithm of $\delta$.
\end{proposition}

 \begin{proof}
     Please see \cref{sec:proof_sqr_rt} for a detailed proof.
 \end{proof}

\paragraph{Closed-Form Time Schedules.}
Beyond the generic rational primitives above, we also showcase some concrete operators that arise as time schedules in diffusion and flow matching models \citep{fu2024unveil,fukumizu2024flow,oko2023diffusion}.
In the typical diffusion setting, the Gaussian forward kernel has mean coefficient $\alpha_t=\exp(-t/2)$ and standard deviation $\sigma_t=\sqrt{1-\exp(-t)}$. 
The next two propositions construct explicit softmax-attention approximators for these two schedules.
We start with the approximation of $\alpha_t$:
\begin{proposition}
[Approximation of $\exp(-t/2)$]
\label{prop:operator1}
Let $C_X > 0$.
Then, for any $\epsilon \in (0,1)$, there exists $f_\alpha \in \calT(H, W, K, C_{KQ}, C_V)$ such that
\begin{align*}
|f_\alpha(t) - \exp(-\frac{t}{2})| \leq \epsilon \quad \text{for all}\quad
0 \leq t \leq C_X.
\end{align*}
For $C_X\geq1$, the network parameters satisfy:
\begin{align*}
                & ~H = O(\ln\delta),
\quad
W = O(\ln \delta),
\quad
K = O(\ln^2 \delta),
\quad C_V = \exp(O(\ln^2 \delta))\cdot O(\sqrt{\ln \delta}), \\
& ~ C_{KQ}
= O(\sqrt{\ln^2\delta\cdot(\ln^2\delta+\ln C_X)}), \quad
\lambda
= \exp(O(\ln^4 \delta))\cdot O(\ln^{\mathrm{polylog}(\delta)}\delta\cdot\ln C_X).
\end{align*}
where $\delta=1/\epsilon$ and $\mathrm{polylog}(\delta)$ denotes a polynomial in the logarithm of $\delta$.
\end{proposition}

\begin{proof}
    Please see \cref{sec:proof_operator1} for a detailed proof.
\end{proof}

We next treat the standard deviation schedule $\sigma_t=\sqrt{1-\exp(-t)}$ in diffusion models.

\begin{proposition}
[Approximation of $\sqrt{ 1- \exp(-t) }$]
\label{prop:operator2}
For any $\epsilon \in (0,1)$ and any $C_X > \epsilon$, there exists $f_\sigma \in \calT(H, W, K, C_{KQ}, C_V)$ such that
\begin{align*}
|f_\sigma(t) - \sqrt{1 - \exp(-t)}| \leq \epsilon \quad \text{for all}
\quad \epsilon \leq t \leq C_X.
\end{align*}
For $C_X\geq1$, the network parameters satisfy:
\begin{align*}
                & ~H = O(\ln^3\delta),
\quad
W = O(\ln^3 \delta),
\quad
K = O(\ln^2 \delta),
\quad 
C_V = \exp(O(\ln^2 \delta))\cdot O(\sqrt{\ln^3 \delta}),\\
& ~  C_{KQ}
= O(\sqrt{\ln^2\delta\cdot(\ln^2\delta+\ln C_X)}), \quad
\lambda
= \exp(O(\ln^4 \delta))\cdot O(\ln^{\mathrm{polylog}(\delta)}\delta\cdot\ln C_X).
\end{align*}
where $\delta=1/\epsilon$, ${\rm polylog}(\delta)$ denotes a polynomial in the logarithm of $\delta$ and $O(\cdot)$ hides polynomial factors depending on $d$ and $n$.
\end{proposition}

\begin{proof}
    Please see \cref{sec:proof_operator2} for a detailed proof.
\end{proof}

These examples demonstrate practical use cases of our translation theorem \cref{thm:maintex_approx_relu_trans}.
Such explicit Transformer constructions may also serve as useful ingredients in future statistical analyses of transformer-based diffusion or flow-matching models.

\section{Proof Sketch of \texorpdfstring{\cref{thm:maintex_approx_relu_trans}}{}}
\label{sec:5proofsketch}

We sketch the proof of \cref{thm:maintex_approx_relu_trans} here and defer detailed proofs to \cref{sec:translation_theorem}. 
The main idea is to translate a ReLU network block-by-block into a softmax-attention network, while controlling both the approximation error and the growth of parameter resources.

\paragraph{Step 1: Attention Primitives.}
The translation relies on constructive approximations of four fundamental components:

\begin{itemize}
\item \textbf{Hardmax Approximation (\cref{lem:bound_on_temperature_to_approximate_hardmax}).}
When the inverse temperature is sufficiently large, the softmax concentrates on the maximal entry: if 
\begin{align*}
    \lambda \geq 
    \frac{\ln(n-1) - \ln \epsilon}{x_1 - \max_{2\leq i\leq n}x_i},
\end{align*}
then 
\begin{align*}
    \|{\rm Softmax}_\lambda(x) - e_1\|_\infty \leq \epsilon.
\end{align*}
This prepares all subsequent constructions.

\item \textbf{Attention-Based Linear Maps (\cref{lem:attn_linear_transfomation,lem:parameter_bound_linear_transformation}).} 
A single-head attention layer, applied to an appropriately augmented input, is able to approximate any column-wise linear transformation $\ell(X) = AXB$ to arbitrary precision.
We derive explicit bounds on $\|W_V\|_F$ and $\|W_{KQ}\|_F$ in terms of $\|A\|_F$, $\|B\|_F$, and the target accuracy.

\item \textbf{Entry-Wise Multiplication (\cref{lem:attention_realize_entrywise_multiplication}).}
A multi-head attention layer is capable of approximating position-dependent element-wise scaling $x_j \mapsto {\rm diag}(v_j)\, x_j$ for prescribed vectors $v_j$.
The required number of heads is $O(n)$.

\item \textbf{Soft-ReLU Approximation (\cref{lem:softrelu_approximate_relu}).}
The function 
\begin{align*}
    {\rm ReLU}_{{\rm soft}}^{(n)}(s) := s \cdot \sigma(\lambda s + \ln n),
\end{align*}
approximates ${\rm ReLU}(s)$ uniformly on $[-C_s, C_s]$ with error $O(\epsilon_{{\rm relu}})$, provided 
\begin{align*}
    \lambda \geq \frac{\ln(C_s n / \epsilon_{{\rm relu}})}{\epsilon_{{\rm relu}}}.
\end{align*}
Crucially, we're capable of realizing this approximation by an attention score computation, since the sigmoid arises naturally from the softmax over two effective logits.
\end{itemize}

\paragraph{Step 2: Simulating a Single ReLU Layer (\cref{lem:approximate_one_layer_relu_with_attention_matrix_new}).}
Using the primitives from Step~1, we construct a three-block attention network $\phi = T \circ {\rm Attn}_3 \circ {\rm Attn}_2 \circ {\rm Attn}_1 \circ P$ that approximates a one-layer ReLU network on a compact domain.
The three layers serve distinct roles:

\begin{enumerate}
\item \textbf{Zooming layer (${\rm Attn}_1$).}
By \cref{lem:attention_realize_entrywise_multiplication}, this layer computes position-dependent scalings ${\rm diag}(w_{i,k,j})\, x_j$ for all hidden-unit indices $(i,k)$ and token positions $j$, stacking the results into a higher-dimensional representation.

\item \textbf{Accumulation layer (${\rm Attn}_2$).}
By \cref{lem:attn_linear_transfomation,lem:parameter_bound_linear_transformation}, this layer forms the pre-activation sums $s_{i,k} = \sum_{j=1}^n w_{i,k,j}^\top x_j$ by implementing the required linear combinations through attention over the augmented sequence.

\item \textbf{Soft-ReLU layer (${\rm Attn}_3$).}
This layer applies the soft-ReLU activation (\cref{lem:softrelu_approximate_relu}) to each $s_{i,k}$ and aggregates the results $\sum_k a_{i,k}\, {\rm ReLU}_{{\rm soft}}^{(n)}(s_{i,k})$ to produce the output.
The key mechanism is that the key-query inner product encodes $\lambda s_{i,k}$, so the softmax weight naturally implements the sigmoid gating.
A suppression constant $C$ in the score matrix ensures that non-target query positions receive negligible contribution.
\end{enumerate}

This construction first produces a vector-valued output for each row index $r \in [d]$ (\cref{lem:approximate_one_layer_relu_with_attention_new}).
To obtain the full matrix output, we introduce selector matrices $Q_r$ and organizer matrices $O_r$ that extract the $r$-th block from the stacked representation and embed each per-row output into the correct block location, respectively.
Absorbing these into the weight matrices yields a valid multi-head attention layer with $H = O(dn^2 N)$ heads (\cref{lem:approximate_one_layer_relu_with_attention_matrix_new}).
Throughout, we track explicit bounds on $H$, $W$, $\lambda$, $\|W_V\|_F$, and $\|W_{KQ}\|_F$.

\paragraph{Step 3: Extending to General Depth (\cref{thm:maintex_approx_relu_trans}).}
To translate a $K_f$-layer ReLU network $f = f^{(K_f)} \circ \cdots \circ f^{(1)}$, we replace each layer $f^{(k)}$ by the three-block attention approximator $g^{(k)}$ from Step~2, with intermediate preprocessing and truncation layers eliminated by appending an identity-preserving head (\cref{rem:removed_affine_layer}).
The end-to-end error is controlled via a recursive bound.
Since each $f_{k}$ is $(W_f B)$-Lipschitz under $\|\cdot\|_\infty$, the composition error satisfies
\begin{align*}
\eta_k \leq W_f B \cdot \eta_{k-1} + \epsilon_k,
\end{align*}
where $\eta_k := \|g_k\circ\cdots \circ g_1(X) - f_k\circ\cdots\circ f_1({\rm vec}(X))\|_{{\rm max}}$.
Setting the per-layer tolerance to 
\begin{align*}
    \epsilon_k = \frac{\epsilon}{K_f (W_f B)^{K_f}},
\end{align*}
ensures $\eta_{K_f} \leq \epsilon$.
Substituting this into the per-layer resource bounds from Step~2 yields the final parameter bounds stated in \cref{thm:maintex_approx_relu_trans}.

\section{Concluding Remarks}
\label{sec:conclusion}
This paper provides a recipe to build transformer approximation theory via a constructive reduction to ReLU approximation theory.
The main result is a translation theorem \cref{thm:maintex_approx_relu_trans}.
This leads to a constructive and quantitative route to approximation theory for softmax Transformers by migrating the mature, target-adaptive approximation theory of ReLU networks.
Specifically, let $\calF$ and $\calT$ denote ReLU and (attention-only) Transformer function classes. \cref{thm:maintex_approx_relu_trans} compiles any ReLU network $f\in\mathcal{F}(K_f,W_f,S,B)$ into a Transformer network $g\in\mathcal{T}(\cdot)$ with $\|g(X)-f(\mathrm{vec}(X))\|_2\le \epsilon$ on a compact domain.
The construction is explicit and tracks $K,W,H,\lambda$ and norm budgets such as $C_{KQ}$ and $C_V$.

Technically, \cref{lem:maintex_approximate_one_layer_relu_with_attention_matrix} simulates one ReLU layer using three attention blocks and quantifies the overhead in head number $H$ and attention (inverse-)temperature $\lambda$.
Unlike universality-only analyses, this reduction inherits target-specific ReLU guarantees whenever they exist.

We show that this translator helps in two complementary ways: universality and target-specific approximation.

\begin{itemize}
    \item For universality, \cref{thm:maintex_attention_uap} gives a Transformer construction for 
    \begin{align*}
        f:[-C_X,C_X]^n\to\mathbb{R}^n, f\in W^{r,\infty}([0,1]^n),
    \end{align*} 
    with explicit rates for $H,W,\lambda,C_{KQ},C_V$ as functions of $\epsilon$.
   This complements \cite[Theorem G.1 in Appendix G]{hu2025universal}. 
   Their result gives a transformer-native, attention-only universality theorem via an existence-type ReLU UAP. 
   Our result makes this route constructive and tracks the resulting resource budgets explicitly.
    
    \item For target-specific bounds, \cref{subsec:transformer_rational_approximation} isolates three primitives --- monomials, reciprocal and min/max clipping ---
    and proves them as \cref{lem:trans_monomial,lem:trans_recipro,lem:min_max_approx}.
    These modules form a parsimonious set for composing  rational function approximators. For example, they cover concrete operators such as $\sqrt{x}$ and schedule terms like $\exp(-t/2)$ in diffusion or flow generative models.
    Together, these examples plot a way to move beyond broad universal approximation and hence sharpen complexity estimates for concrete operator classes.
    They provide more informative tools for analyzing modern transformer backbones.
\end{itemize}

\paragraph{Practicality.}
Throughout, we define ``transformers'' as attention-only softmax networks to isolate attention’s approximation cost. 
Standard transformer blocks add residual connections and position-wise FFNs and hence subsume this model class.
From an approximation standpoint, residual additions are algebraic: if a block outputs $X+F(X)$, we set $F$ to approximate $g(X)-X$. 
Our proof already constructs identity-preserving attention heads, so forming $g(X)-X$ is compatible with our construction. 
Similarly, including FFNs is straightforward: a standard block can disable the FFN, or treat it as additional ReLU capacity consistent with the ReLU approximator class that our translation theorem compiles from.
Therefore, our theorems apply to practical blocks, up to constant bookkeeping in resources.

\paragraph{Future Directions.}
Several directions remain open.
On the theory side, our bounds make clear that the main overhead of simulating ReLU-style gating concentrates in the number of heads $H$ and the inverse temperature $\lambda$.
Tightening these dependencies, and establishing matching lower bounds for specific targets, would further clarify the intrinsic cost of softmax attention approximation.
On the modeling side, extending the translation recipe to more realistic transformer components (e.g., residual connections, normalization, causal masking, cross-attention, and variable-length sequences) would broaden applicability.
Finally, connecting these constructive approximations to learning-theoretic questions is a promising avenue for future work.
Examples include statistical rates, norm-controlled generalization, and the approximation of operators encountered in generative modeling pipelines.

\section*{Impact Statement}
By the theoretical nature of this work, we do not anticipate any negative social impact.

\section*{Acknowledgments}
The authors thank Hong-Yu Chen, Mimi Gallagher, Sara Sanchez, T.Y. Ball, Dino Feng and Andrew Chen for valuable conversations.
JH also thanks Hude Liu, Hong-Yu Chen, Po-Chaio Lin, Maojiang Su, Weimin Wu for collaborations on related topics. 

JH is partially supported by Northwestern University’s Walter P. Murphy Fellowship and Terminal Year Fellowship (Paul K. Richter Memorial Award).
Han Liu is partially supported by NIH R01LM1372201, NSF
AST-2421845, Simons Foundation
MPS-AI-00010513, AbbVie , Dolby and Chan Zuckerberg Biohub Chicago Spoke Award.
The content is solely the responsibility of the authors and does not necessarily represent the official
views of the funding agencies.

\newpage
\appendix
\label{sec:append}
\part*{Appendix}
{
\setlength{\parskip}{-0em}
\startcontents[sections]
\printcontents[sections]{ }{1}{}
}

{
\setlength{\parskip}{-0em}
\startcontents[sections]
\printcontents[sections]{ }{1}{}
}

\clearpage
\section{Proof of Translation Theorem}\label{sec:translation_theorem}
In this appendix, we provide a complete proof of the translation theorem.
Our argument is fully constructive: starting from a given ReLU network, we explicitly build a softmax-attention network that approximates the same target function on a prescribed compact domain while tracking the resulting architectural and norm bounds.

The proof proceeds in three stages.
We first develop several attention primitives that will serve as reusable building blocks.
We then combine these primitives to simulate a single ReLU layer by a constant-depth attention module. 
Finally, we compose these layer-wise constructions to obtain the approximation result for general-depth ReLU networks.

\textbf{Organization}.
In \cref{subsec:auxiliary_lemmas_translation_theory}, we introduce the auxiliary lemmas used throughout the proof.
\cref{subsec:main_proof_one_layer_relu} establishes the one-layer translation result.
Finally, \cref{subsec:main_proof_multi_layer_relu} lifts this construction to multi-layer ReLU networks and completes the proof of the main theorem.

\subsection{Auxiliary Lemmas}\label{subsec:auxiliary_lemmas_translation_theory}

This subsection collects the technical ingredients needed for the proof of the translation theorem \cref{lem:maintex_approximate_one_layer_relu_with_attention_matrix}-\cref{thm:maintex_approx_relu_trans}. 
Each lemma isolates one basic operation required in the simulation of a ReLU layer by softmax attention.

\paragraph{Hardmax Approximation by Finite Temperature Softmax.}
The softmax function with a sufficiently large inverse temperature $\lambda$ concentrates its mass near the coordinate of the maximal entry. 
The following lemma makes this precise and provides the quantitative threshold on $\lambda$ required to achieve a given approximation
accuracy. 

\begin{lemma}[Approximating Hardmax with Finite-Temperature Softmax, Lemma D.1 of \cite{hu2025universal}]\label{lem:bound_on_temperature_to_approximate_hardmax}
Let $x=[x_1,x_2,\dots,x_n]\in\R^n$ and $\epsilon>0$.
Define softmax function with temperature $\Softmax_\lambda(\cdot)$ as
\begin{align*}
\Softmax_\lambda(x):=
[\frac{\exp(\lambda x_1)}{\sum_{j=1}^n\exp(\lambda x_j)},\dots,\frac{\exp(\lambda x_n)}{\sum_{j=1}^n\exp(\lambda x_j)}].
\end{align*}
Assume $x_1=\max_{1\leq i \leq n}x_i$ is unique, and $x_2=\max_{2\leq i \leq n}x_i$.
Then, if $\lambda\geq(\ln(n-1)-\ln(\epsilon))/(x_1-x_2)$, we have:
\begin{align*}
\|\Softmax_\lambda(x)-e_1\|_{\rm max}\le\epsilon.
\end{align*}
where $e_1 \in \R^n$ is the one-hot vector supported on the index of the maximum.
\end{lemma}

\begin{proof}
See proof of \cite[Lemma D.1]{hu2025universal}.
\end{proof}

\paragraph{Attention Based Simulation of Linear Maps.}
Given the hardmax primitive, we next show that a single-head attention layer is able to simulate any column-wise linear transformation $X \rightarrow AXB$. 

\begin{lemma}[Attention simulates column-wise linear transformations, Lemma G.1 of \cite{hu2025universal}]\label{lem:attn_linear_transfomation}
Let $X \in \R^{d\times n}$ and define
\begin{align*}
\ell(X) := AXB \in \R^{d_{\rm out}\times n},~ A \in \R^{d_{\rm out}\times d},~ B \in \R^{n\times n}
\end{align*}
be a linear map.
Assume that all entries of $B$ are larger than $1$.
Consider the augmented input
\begin{align*}
Z := \begin{bmatrix}
X & 0_d\\
I_n & 0_n\\
0_{1\times n} & 1
\end{bmatrix} \in \R^{(d+n+1)\times(n+1)},
\end{align*}
where $0_d \in \R^{d\times 1}$ be the all-zeros vector.
Then for any $\epsilon>0$, there exists a single-head attention
\begin{align*}
    {\rm Attn}(Z)=W_VZ\cdot\Softmax((W_KZ)^{\top}(W_QZ))
\end{align*}
such that
\begin{align*}
    \|{\rm Attn}(Z)-
    \begin{bmatrix}
    \ell(X) & 0_{d+n+1}
    \end{bmatrix}
    \|_{\rm max}\leq \epsilon.
\end{align*}
\end{lemma}

\begin{proof}
    See proof of \cite[Lemma G.1]{hu2025universal}.
\end{proof}

\paragraph{Explicit Norm Bounds for Linear-Map Simulation.}
While \cref{lem:attn_linear_transfomation} asserts existence, the following companion lemma provides the explicit norm bounds on $W_V$ and $W_{KQ} := W_K^\top W_Q$.

\begin{lemma}\label{lem:parameter_bound_linear_transformation}
Suppose $\|X\|_\infty\leq C_X$.
Let ${\rm Attn}$ be the single-head attention based approximator constructed in \cref{lem:attn_linear_transfomation}, satisfying
\begin{align*}
    {\rm Attn}(Z)=W_VZ\cdot\Softmax((W_KZ)^{\top}(W_QZ)).
\end{align*}
Then parameter bounds of the constructed attention network follow:
\begin{align*}
    \|W_V\|_F = & ~O(\sqrt{n}\|A\|_F\|B\|_F), \\ 
    \|W_{KQ}\|_F = & ~ O(\|B\|_F\sqrt{n\ln(nC_X\|A\|_{\rm max}\|B\|_F/\epsilon)}).
\end{align*}
\end{lemma}

\begin{proof}
The proof is based on proof of \cite[Lemma G.1]{hu2025universal}.
To avoid overlap, we do not repeat the construction detail.

Let $B_{ij}$ denote the $(i,j)$-th entry of $B$.
We set
\begin{align*}
    s_i:=\sum_{j=1}^nB_{ij},~ S:=(s_1,\dots,s_n)\in\R^n
\end{align*}
and let
\begin{align*}
    M:=\max_{i\in [n]}s_i.
\end{align*}

Then $M\leq \sqrt{n}\|B\|_F$.
Recall that in proof of \cite[Lemma G.1]{hu2025universal}, the constructed weight matrices take the form:
\begin{align*}
    W_V := & ~
    \begin{bmatrix}
    3MA & 0_{d_{\rm out}\times (n+1)}
    \end{bmatrix},\\
    W_K := & ~
    \begin{bmatrix}
    0_{n\times d} & \ln(B^\top) & \ln(3M\one_n - S^\top)
    \end{bmatrix},\\
    W_Q := & ~
    \begin{bmatrix}
        0_{n\times d} & I_n & T\one_n
    \end{bmatrix},
\end{align*}
where $T$ is a parameter to be determined later.
Then we have:
\begin{align}
    \|W_V\|_F = & ~3M\|A\|_F,\notag \\
    \|W_K\|_F \leq & ~ \sqrt{\|B\|^2_F + n(\ln(2M))^2 }, \annot{$0 \leq \ln B_{ij} \leq B_{ij}$}\\
    \|W_Q\|_F = & ~ \sqrt{T^2+1}n.\label{eq:inplicit_form_of_matrix_norm_layer_2}
\end{align}
The overall approximation error has upper bound
\begin{align*}
    \|{\rm Attn}(Z)-
    \begin{bmatrix}
    \ell(X) & 0_{d+n+1}
    \end{bmatrix}
    \|_{\rm max} \leq 3Mn\|AX\|_{\rm max}\delta,
\end{align*}
where $\delta$ is the approximation error of padding layer.
Then it suffices to let $\delta \leq \frac{\epsilon}{3Mn\|AX\|_{\rm max}}$ to ensure the overall error is less than $\epsilon$.
We further determine the value of $T$ to bound $\delta$.
Define that
\begin{align*}
    H(i) := & ~
    \begin{cases}
        T\sum_{r=1}^n\ln B_{ir}, & ~1 \leq i \leq n, \\
        T\sum_{r=1}^n\ln(3M-s_r), & ~ i=n+1.
    \end{cases} \\
    W(i) := & ~ (\Softmax H)_i.
\end{align*}
Then by \cite{hu2025universal}, $\delta$ has upper bound
\begin{align*}
    \delta \leq 1 - W(n+1)
\end{align*}
Since for any $i\in [n]$, $H(n+1) - H(i) \geq Tn\ln2$, by \cref{lem:bound_on_temperature_to_approximate_hardmax} we get it suffices to let $T=O(\sqrt{\frac{\ln(Mn\|AX\|_{\rm max}/\epsilon)}{n}})$ to ensure $\delta \leq \frac{\epsilon}{3Mn\|AX\|_{\rm max}}$.
We then simplify \eqref{eq:inplicit_form_of_matrix_norm_layer_2} with the expression of $T$ and $\delta$ :
\begin{align*}
    \|W_V\|_F \lesssim & ~
    \sqrt{n}\|A\|_F\|B\|_F, \\
    \|W_K\|_F \lesssim & ~
    \|B\|_F, \\
    \|W_Q\|_F \lesssim & ~
    \sqrt{n\ln(nC_X\|A\|_{\rm max}\|B\|_F/\epsilon)}, \\
    \|W_KW_Q\|_F \lesssim & ~
    \|B\|_F\sqrt{n\ln(nC_X\|A\|_{\rm max}\|B\|_F/\epsilon)}.
\end{align*}
This completes the proof.
\end{proof}

\begin{remark}
Since any matrix $B \in \R^{n \times n}$ can be written as $B = B_1 - B_2$ with $B_1, B_2$ having all entries greater than $1$, the norm bounds in \cref{lem:attn_linear_transfomation} extend to arbitrary $B$ by using a two-head attention network, one head for $B_1$ and one for $B_2$.
\end{remark}

\paragraph{Attention Based Entry-wise Multiplication.}
The linear-map primitive handles affine operations, but simulating a ReLU layer also requires position-dependent scaling, i.e., computing $\diag(v_j)\,x_j$ for prescribed vectors $v_j$ at each position $j$. 
The following lemma shows that an $(n+1)$-head attention layer is capable of performing all $n$ such scalings in parallel.

\begin{lemma}[Approximation Theory of Entry-wise Multiplication through Attention Network]\label{lem:attention_realize_entrywise_multiplication}
Fix $n\in\mathbb{N}$ and $d\in\mathbb{N}$.
Let $X=[x_1,\ldots,x_n]\in\R^{d\times n}$.
Suppose the preprocessed input sequence is in the form
\begin{align*}
    X_p := \begin{bmatrix} X & 0_d \\ I_n & 0_n \\ 0_{1\times n} & 1 \end{bmatrix} \in \R^{(d+n+1)\times(n+1)},
\end{align*}
where $X$ is the input sequence on a bounded domain, satisfying $\|X\|_{\infty}\leq C_X$.
Let $v_1,\ldots,v_n\in\R^d$ be arbitrary vectors satisfying $\max_{i\in[n]}\|v_i\|_\infty\leq B_V$.
We define the padded block $Y(X)$ as
\begin{align*}
    Y(X) := \begin{bmatrix} 
    {\rm diag}(v_1)x_1 & \cdots & {\rm diag}(v_n)x_n & 0_d \end{bmatrix} \in \R^{d\times(n+1)}.
\end{align*}
Further, we define the target output $Z(X)$ as
\begin{align*}
    Z(X):= \begin{bmatrix}
    Y(X) \\
    I_{n+1}
    \end{bmatrix}\in\R^{(d+n+1)\times(n+1)}.
\end{align*}
Then for any $0<\epsilon<1$, there exists an $(n+1)$-head self-attention layer ${\rm Attn}_v$ with temperature $\lambda$ such that for all such $X$,
\begin{align*}
\|{\rm Attn}_v(X_p)-Z(X)\|_{\infty} \leq \epsilon .
\end{align*}
Moreover, the parameter bound of constructed network satisfies 
\begin{align*}
    H=O(n),~\lambda=O(\ln\frac{C_XB_Vn}{\epsilon}),~\|W_K\|_F= O(\sqrt{n}), ~\|W_Q\|_F= O(n),~\|W_V\|_F= O(\sqrt{d}B_V+\sqrt{n}).
\end{align*}
\end{lemma}

\begin{proof}
We present the construction first, then compute the parameter bound by analyzing the error propagation.

\textbf{Construction.}

We define the position selector $S_{\rm pos}$ as 
\begin{align*}
    S_{\rm pos}:=\begin{bmatrix}
        0_{(n+1)\times d} & I_{n+1}
    \end{bmatrix}\in\R^{(n+1)\times(d+n+1)}.
\end{align*}
Then $S_{\rm pos}X_p=I_{n+1}$.

For each $i\in[n]$, we construct an attention head ${\rm Attn}^{(i)}$ as follows.

\begin{itemize}
\item \textbf{Keys.}
We set keys matrices $W_K^{(i)}$ as $W_K=S_{\rm pos}$.
\item \textbf{Queries.}
We define the routine matrices $M^{(i)}$ as
\begin{align*}
    M^{(i)}=\begin{bmatrix}
    e_{n+1} & \dots & e_{n+1} & e_i & e_{n+1} & \dots & e_{n+1}
    \end{bmatrix}\in\R^{(n+1)\times(n+1)},
\end{align*}
where the $i$-th column equals $e_i$ and every other column equals $e_{n+1}$.
We set $W_Q^{(i)}=M^{(i)}S_{\rm pos}$.
\item \textbf{Values.}
We set
\begin{align*}
    W_V^{(i)}=\begin{bmatrix}
        {\rm diag}(v_i) & 0_{d\times(n+1)} \\
        0_{(n+1)\times d} & 0_{(n+1)\times(n+1)}
    \end{bmatrix}.
\end{align*}
\end{itemize}
Altogether, the output of attention head ${\rm Attn}^{(i)}$ is
\begin{align*}
    {\rm Attn}^{(i)}(X_p)=\begin{bmatrix}
        {\rm diag}(v_i)X & 0_{d\times1} \\
        0_{(n+1)\times n} & 0_{(n+1)\times 1}
    \end{bmatrix}\Softmax(\beta M^{(i)}).
\end{align*}
We add an additional attention head  ${\rm Attn}^{n+1}$to preserve the $I_{n+1}$ structure in $Z(X)$.
Specifically, we let $W_K^{(n+1)}=W_Q^{(n+1)}=S_{\rm pos}$, and 
\begin{align*}
    W_V^{(n+1)}=\begin{bmatrix}
        0_{d\times(d+n+1)}\\
        S_{\rm pos}
    \end{bmatrix}.
\end{align*}
Then we have:
\begin{align*}
    {\rm Attn}^{(n+1)}(X_p)=\begin{bmatrix}
        0_{d\times(n+1)}\\
        I_{n+1}
    \end{bmatrix}\Softmax(\lambda I_{n+1}).
\end{align*}
Finally, we construct the ${\rm Attn}_v$ as
\begin{align*}
    {\rm Attn}_v(X_p)=\sum_{i=1}^{n+1}{\rm Attn}^{(i)}(X_p).
\end{align*}
\textbf{Error Analysis and Parameter Bound}

Note that the last $(n+1)$ rows of $\sum_{i=1}^n{\rm Attn}^{(i)}(X_p)$ are zero, while the first $d$ rows of ${\rm Attn}^{(n+1)}(X_p)$ are zero. 
Hence the two error contributions do not overlap in rows.
Therefore, we bound error of $\sum_{i=1}^n{\rm Attn}^{(i)}(X_p)$ and ${\rm Attn}^{(n+1)}(X_p)$ separately.
\begin{itemize}
\item \textbf{Error Analysis of $\sum_{i=1}^n{\rm Attn}^{(i)}(X_p)$.}

We denote $\begin{bmatrix}
        {\rm diag}(v_i)X & 0_{d\times1} \\
        0_{(n+1)\times n} & 0_{(n+1)\times 1}
    \end{bmatrix}$ by $V^{(i)}$ for convenience.
By \cref{lem:bound_on_temperature_to_approximate_hardmax}, it suffices to let $\lambda=O(\ln(n/\epsilon))$ to ensure
\begin{align*}
    \|\Softmax(\lambda M^{(i)})-M^{(i)}\|_\infty\leq\epsilon_0.
\end{align*}
Then we have:
\begin{align*}
    & ~ \|\sum_{i=1}^n{\rm Attn}^{(i)}(X_p)-\begin{bmatrix}
        Y(X) \\
        0_{(n+1)\times(n+1)}
    \end{bmatrix}\|_{\rm max} 
    \\
    \leq & ~
    \sum_{i=1}^n\|{\rm Attn}^{(i)}(X_p)-V^{(i)}M^{(i)}\|_{\rm max} \annot{$\begin{bmatrix}
        Y(X) \\
        0_{(n+1)\times(n+1)}
    \end{bmatrix}=\sum_{i=1}^nV^{(i)}M^{(i)}$}\\
    \leq & ~
    \sum_{i=1}^nC_XV\epsilon \annot{$\|X\|_{\rm max}\leq C_X,\max_{i\in[n]}\|v_i\|_{\rm max}\leq B_V$}\\
    = & ~
    nC_XV\epsilon.
\end{align*}
\item \textbf{Error Analysis of ${\rm Attn}^{(n+1)}(X_p)$.}

By \cref{lem:bound_on_temperature_to_approximate_hardmax}, it suffices to let $\lambda=O(\ln(n/\epsilon))$ to ensure
\begin{align*}
    \|\Softmax(\lambda I_{n+1})-I_{n+1}\|_{\rm max}\leq\epsilon.
\end{align*}
Then we have:
\begin{align*}
    \|{\rm Attn}^{(n+1)}(X_p)-\begin{bmatrix}
    0_{d\times(n+1)}\\
    I_{n+1} 
    \end{bmatrix}\|_{\rm max}\leq\epsilon.
\end{align*}
\end{itemize}
Altogether, we have if $\lambda=O(\ln(n/\epsilon))$, then
\begin{align*}
    \|{\rm Attn}_v(X_p)-Z(X)\|_{\rm max}=\|\sum_{i=1}^{n+1}{\rm Attn}^{(i)}(X_p)-\begin{bmatrix}
    Y(X) \\
    I_{n+1}
    \end{bmatrix}\|_{\rm max}\leq(nC_XB_V+1)\epsilon.
\end{align*}
Then it suffices to let $\lambda=O(\ln\frac{C_XB_Vn}{\epsilon})$ to ensure total error is less than $\epsilon$.
By construction of $W_K^{(i)},W_Q^{(i)},W_V^{(i)}$, we have:
\begin{align*}
    H=O(n),~\|W_K\|_F= O(\sqrt{n}), ~\|W_Q\|_F= O(n),~\|W_V\|_F= O(\sqrt{d}B_V+\sqrt{n}).
\end{align*}
This completes the proof.
\end{proof}

\paragraph{Soft-ReLU Approximation via Attention Scores.}
To simulate the ReLU nonlinearity, we exploit the observation that the sigmoid function arises naturally from the softmax.
Specifically, we define a smooth approximation to ReLU that is expressible as an attention score computation, and bound its approximation error in terms of the inverse temperature.

\begin{lemma}[Soft-ReLU Approximation]\label{lem:softrelu_approximate_relu}
Fix $n \geq 1$, $C_s > 0$ and $\epsilon_{{\rm relu}} \in (0,1)$.
Define
\begin{align*}
  {\rm ReLU}_{\rm soft}^{(n)}(s) := s \cdot \frac{ne^{\lambda s}}{ne^{\lambda s}+1} = s \cdot \sigma(\lambda s + \ln n),
\end{align*}
where $\sigma(t) = e^t/(e^t+1)$ is the sigmoid function.
If
\begin{align*}
  \lambda \geq \frac{\ln(C_s n/\epsilon_{{\rm relu}})}{\epsilon_{{\rm relu}}},
\end{align*}
then for all $|s| \leq C_s$:
\begin{align*}
  |{\rm ReLU}(s) - {\rm ReLU}_{\rm soft}^{(n)}(s)| \leq 2\epsilon_{{\rm relu}}.
\end{align*}
\end{lemma}

\begin{proof}
We fix $\delta = \epsilon_{{\rm relu}}$ and consider three cases.

\textbf{Case 1: $s \geq \delta$.}
\begin{align*}
  |{\rm ReLU}(s) - {\rm ReLU}_{\rm soft}^{(n)}(s)|
  = s \cdot \frac{1}{ne^{\lambda s}+1}
  \leq s \cdot n^{-1}e^{-\lambda s}
  \leq C_s \cdot n^{-1} e^{-\lambda\delta}
  \leq \epsilon_{{\rm relu}},
\end{align*}
where the last inequality uses $\lambda\delta \geq \ln(C_sn/\epsilon_{\rm relu})$.

\textbf{Case 2: $s \leq -\delta$.}
\begin{align*}
  |{\rm ReLU}_{\rm soft}^{(n)}(s)|
  = |s| \cdot \frac{ne^{\lambda s}}{ne^{\lambda s}+1}
  \leq |s| \cdot ne^{\lambda s}
  \leq C_s \cdot n \cdot e^{-\lambda\delta}
  \leq \epsilon_{{\rm relu}},
\end{align*}
where we used $e^{\lambda s} \leq e^{-\lambda\delta}$ for $s \leq -\delta$ and $\lambda\delta \geq \ln(C_s n/\epsilon_{{\rm relu}})$.

\textbf{Case 3: $|s| < \delta$.}
Since $0 \leq \sigma(\cdot) \leq 1$ and ${\rm ReLU}(s) \leq \delta$:
\begin{align*}
  |{\rm ReLU}(s) - {\rm ReLU}_{\rm soft}^{(n)}(s)| \leq {\rm ReLU}(s) + |{\rm ReLU}_{\rm soft}^{(n)}(s)| \leq \delta + \delta = 2\delta = 2\epsilon_{{\rm relu}}.
\end{align*}
Combining all three cases completes the proof.
\end{proof}

\paragraph{Single-Layer Translation: Vector Output Case.}
With all four primitives in place, we now assemble them into a three-block attention network that approximates a single-layer ReLU network with vector output. 
The three blocks serve distinct roles: a zooming layer computes position-dependent scalings, an accumulation layer forms the pre-activation sums via linear combinations, and a soft-ReLU layer applies the
approximate nonlinearity and aggregates the results. 
This construction is the core building block for the full translation theorem.

\begin{lemma}\label{lem:approximate_one_layer_relu_with_attention_new}
Define $\mathcal{X}\subset\R^{d\times n}$ to be a compact domain of input sequences.
Let $f:=(f_1,\dots,f_n), ~f:\mathcal{X}\rightarrow\R^n$ be a ReLU network with the form
\begin{align*}
    f_i=\sum_{k=1}^N a_{i,k}{\rm ReLU}(\sum_{j=1}^nw_{i,k,j}^\top x_j),~ a_{i,k}\in\{-1,1\},~w_{i,k,j}\in\R^d.
\end{align*}
Suppose $\|w_{i,k,j}\|_\infty\leq w_0$ for all $i,j \in [n], k \in [N]$.
Then for any $\epsilon \in (0,1)$, there exists a network $\phi$ composed of multi-head attention layers such that for all $X\in[-C_X,C_X]^{d\times n}$, it holds:
\begin{align*}
    \|\phi(X)-f(X)\|_{\rm max} \le \epsilon.
\end{align*}
More specifically, $\phi$ takes the form
\begin{align*}
    \phi = T \circ {\rm Attn}_3 \circ {\rm Attn}_2 \circ {\rm Attn}_1 \circ P,
\end{align*}
where $P$ is the preprocessing layer that pads a zero token with positional coding and $T$ is the truncation layer that deletes the last token.
Further, For $C_X\geq1$ the parameter bounds of constructed network follow
\begin{align*}
  & ~H = O(n^2N),\quad W=O(dnN), \quad
  \lambda = O(\frac{N\ln(dnNC_Xw_0/\epsilon)}{\epsilon}),\\ 
  & ~ \|W_V\|_F = w_0\sqrt{dN}+n\sqrt{nd}, \quad \|W_{KQ}\|_F=O(n\sqrt{n\ln\frac{nN C_X w_0}{\epsilon}}).
\end{align*}
\end{lemma}

\begin{proof}

\textbf{Preprocessing}

We define the preprocessing layer as
\begin{align*}
    P(X) := 
    \begin{bmatrix}
        X & 0_n \\
        I_n & 0_n \\
        0_{1\times n} & 1
    \end{bmatrix}.
\end{align*}

\textbf{First Layer: Zooming Layer}

In this layer, we aim to obtain an output consisting of position-dependent element-wise scaling of $x_j$ by the weight vector $w_{i,k,j}$ for each $(i,k)\in\R^n\times\R^N$.
We apply \cref{lem:attention_realize_entrywise_multiplication} to achieve this.

Concretely, by \cref{lem:attention_realize_entrywise_multiplication}, for each $(i,k)$, there exists an $(n+1)$-head self-attention layer ${\rm Attn}_{i,k}^{(1)}$ with temperature $\beta$ such that
\begin{align*}
  \|{\rm Attn}_{i,k}^{(1)}(X_p) - Z_{i,k}(X)\|_{\rm max} \leq \epsilon_1,
\end{align*}
where the target output is
\begin{align*}
  Z_{i,k}(X) :=
  \begin{bmatrix}
    Y_{i,k}(X) \\ I_{n+1}
  \end{bmatrix},
  \quad
  Y_{i,k}(X) := \begin{bmatrix}
{\rm diag}(w_{i,k,1})x_1,&\ldots,&{\rm diag}(w_{i,k,n})x_n,&0_d      
  \end{bmatrix}
  \in \R^{d\times(n+1)}.
\end{align*} 
We stack all $nN$ zooming operations into a single multi-head layer ${\rm Attn}_1$.
For each $(i,k)$-th zooming operation, define the organizer matrix
\begin{align*}
  O_{i,k} :=
  \begin{bmatrix}
    0_{((i-1)N+k-1)d \times d} \\
    I_d \\
    0_{(nN-(i-1)N-k)d \times d}
  \end{bmatrix}
  \in \R^{dnN \times d},
\end{align*}
which places the $d$-dimensional output of the $(i,k)$-th zooming into the correct block of a stacked $dnN$-dimensional vector.

We use the same key and query matrix construction as in \cref{lem:attention_realize_entrywise_multiplication} and the position selector $S_{{\rm pos}} := [0_{(n+1)\times d},~I_{n+1}] \in \R^{(n+1)\times(d+n+1)}$:
\begin{align*}
  W_K^{(i,k,j)} = S_{{\rm pos}},\quad W_Q^{(i,k,j)} = M^{(j)}S_{{\rm pos}}.
\end{align*}
The value matrix routes the scaled output to the $(i,k)$-th block:
\begin{align*}
  W_V^{(i,k,j)} :=
  \begin{bmatrix}
    O_{i,k}~{\rm diag}(w_{i,k,j}) & 0_{dnN\times(n+1)} \\
    0_{(n+1)\times d} & 0_{(n+1)\times(n+1)}
  \end{bmatrix}
  \in \R^{(dnN+n+1)\times(d+n+1)}.
\end{align*}

We add one additional identity-preserving head to maintain the positional encoding block $I_{n+1}$ in the last $(n+1)$ rows, following construction in \cref{lem:attention_realize_entrywise_multiplication}.

The output is approximately
\begin{align*}
  Z^{(1)} \approx
  \begin{pmatrix}
    Y_{1,1}(X) \\ Y_{1,2}(X) \\ \vdots \\ Y_{n,N}(X) \\ I_{n+1}
  \end{pmatrix}
  \in \R^{(dnN+n+1)\times(n+1)},
\end{align*}
up to error $\epsilon_1$.
From \cref{lem:attention_realize_entrywise_multiplication} we have the parameter bound:
\begin{align}\label{eq:parameter_bound_new_1}
  H_1 = O(n^2N),\quad
  \lambda_1 = O(\ln\frac{C_X w_0 n N}{\epsilon_1}),\quad
  \|W_V^{(1)}\|_F = O(\sqrt{d}~w_0\sqrt{N}+\sqrt{n}),\quad
  \|W_{KQ}^{(1)}\|_F = O(n^{3/2}).
\end{align}

\textbf{Second Layer: Accumulation Layer}

In this layer, we aim to obtain an output consisting of inner products' accumulation 
\begin{align*}
    s_{i,k}:=\sum_{j=1}^nw_{i,k,j}^\top x_j=\sum_{j=1}^n\sum_{m=1}^d(w_{i,k,j}^\top)_m (x_j)_m.
\end{align*}
Noticing that $s_{i,k}$ is a linear combination of $x_j$, we apply \cref{lem:attn_linear_transfomation} to construct the layer.

We use $X' \in \R^{dnN\times n}$ to denote the data block of $Z^{(1)}$ (first $dnN$ rows, first $n$ columns).

Define the extraction row vector $A_{i,k}\in\R^{1\times dnN}$ by
\begin{align*}
(A_{i,k})_\ell
:=
\begin{cases}
1, & \ell \in \{((i-1)N+k-1)d+1,\dots,((i-1)N+k)d\},\\
0, & \text{otherwise},
\end{cases}
\end{align*}
for each $(i,k)\in\R^n\times\R^N$.
Then $A_{i,k}$ selects the $d$ rows corresponding to the $(i,k)$-th block, and
\begin{align*}
  A_{i,k}X'1_n = \sum_{j=1}^n \sum_{m=1}^d (w_{i,k,j})_m(x_j)_m = s_{i,k}.
\end{align*}
Writing $B := 1_n1_n^\top \in \R^{n\times n}$, we have:
\begin{align*}
  A_{i,k}~X'~B = s_{i,k}\cdot\one_n^\top \in \R^{1\times n}.
\end{align*}
This is exactly the $AX'B$ form required by \cref{lem:attn_linear_transfomation}.
Therefore, \cref{lem:attn_linear_transfomation} guarantees that there exists a single-head attention ${\rm Attn}_2^{(i,k)}$ such that 
\begin{align*}
    \|{\rm Attn}_2^{(i,k)}(Z^{(1)})-s_{i,k}\one_n^\top\|_{\rm max}\leq \epsilon_2.
\end{align*}
Using organizer matrices analogous to Layer 1, each $(i,k)$-head places the scalar $s_{i,k}$ into row $(i-1)N+k$ of the stacked representation. 
Same as in Layer 1, an additional identity-preserving head maintains $I_{n+1}$ in the last $n+1$ rows.
The output $Z^{(2)}$ is approximately:
\begin{align*}
  Z^{(2)} \approx
  \begin{bmatrix}
    s_{1,1}\cdot\one_n^\top & 0 \\
    s_{1,2}\cdot\one_n^\top & 0 \\
    \vdots & \vdots \\
    s_{n,N}\cdot\one_n^\top & 0 \\
    I_{n+1} &
  \end{bmatrix}
  \in \R^{(nN+n+1)\times(n+1)},
\end{align*}
up to error $\epsilon_2$.
From \cref{lem:bound_on_temperature_to_approximate_hardmax} and \cref{lem:parameter_bound_linear_transformation} we have:
\begin{align}\label{eq:parameter_bound_new_2}
  H_2 = O(nN),\quad \lambda_2= O(\ln\frac{n}{\epsilon_2}),\quad
  \|W_V^{(2)}\|_F = O(n\sqrt{nd}),\quad 
  \|W_{KQ}^{(2)}\|_F = O(n\sqrt{n\ln\frac{n^2 C_X w_0}{\epsilon_2}}).
\end{align}

\textbf{Third Layer: Soft-ReLU Approximation Layer}

In this layer, our goal is to construct a multi-head attention layer ${\rm Attn}_3$ with temperature $\lambda_3$ such that the output approximates the $1\times(n+1)$ row vector
\begin{align*}
  (f_1,~ f_2,~ \dots,~ f_n,~ 0),
  \quad f_i = \sum_{k=1}^N a_{i,k}{\rm ReLU}(s_{i,k}).
\end{align*}

For head $(i,k)$, we set
\begin{align*}
  W_{KQ}^{(i,k)} := e_{(i-1)N+k}e_{nN+i}^{\top} + C\cdot e_{nN+n+1}v_i^{\top} \in \R^{(nN+n+1)\times(nN+n+1)},
\end{align*}
where $v_i := \sum_{j \ne i} e_{nN+j} + e_{nN+n+1} \in \R^{nN+n+1}$ selects all positional-encoding entries except the $i$-th, and $C > 0$ is a suppression constant to be determined.

We evaluate the attention score $\lambda_3(z_p^{(2)})^{\top} W_{KQ}^{(i,k)}z_q^{(2)}$ computed with temperature $\lambda_3$ between key column $p$ and query column $q$ for each case.

\begin{itemize}
\item \textbf{The target query position query $q = i$.}
  
Since the $i$-th positional-encoding entry of $z_i^{(2)}$ equals $1$, we have $v_i^{\top} z_i^{(2)} = 0$.
Therefore, for data key $p \leq n$, score $= \lambda_3 \cdot Z^{(2)}_{(i-1)N+k,p} = \lambda_3 \cdot s_{i,k}$.
For reference key $p = n+1$, score $= \lambda_3 \cdot 0 = 0$.
The resulting softmax weights are:
  \begin{align*}
    \alpha_p =
    \begin{cases}
      \dfrac{e^{\lambda_3 \cdot s_{i,k}}}{ne^{\lambda_3 \cdot s_{i,k}}+1}, & p \leq n, \\[6pt]
      \dfrac{1}{ne^{\lambda_3 \cdot s_{i,k}}+1}, & p = n+1.
    \end{cases}
  \end{align*}

\item \textbf{Non-target query positions query $q \ne i$.}
  
Since only the $i$-th positional-encoding entry of $z_q^{(2)}$ equals $0$, we have $v_i^{\top} z_q^{(2)} = 1$.
Therefore for data key $p \leq n$, score $= \lambda_3 \cdot C \cdot (z_p^{(2)})_{nN+n+1} = 0$.
For reference key $p = n+1$, score $= \lambda_3 \cdot C$.
The softmax is dominated by the reference key when $C$ is large, because the reference key weight equals $\frac{e^{\lambda_3 C}}{n + e^{\lambda_3 C}}$.
\end{itemize}

We set the value matrix as
\begin{align*}
  W_V^{(i,k)} := a_{i,k}(e_{(i-1)N+k}^{(nN+n+1)})^{\top} \in \R^{1 \times (nN+n+1)}.
\end{align*}
Then for key $p \leq n$: $W_V^{(i,k)}z_p^{(2)} = a_{i,k}s_{i,k}$, and for key $p = n+1$: $W_V^{(i,k)}z_{n+1}^{(2)} = 0$.

For query $q = i$, the scalar output at $i$-th column is:
\begin{align*}
  ({\rm Attn}^{(i,k)}(Z^{(2)}))_{1,i}
  &= \sum_{p=1}^{n} a_{i,k}s_{i,k}\cdot\frac{e^{\lambda_3 \cdot s_{i,k}}}{ne^{\lambda_3 \cdot s_{i,k}}+1} + 0 \\
  &= a_{i,k}s_{i,k}\cdot\frac{ne^{\lambda_3 \cdot s_{i,k}}}{ne^{\lambda_3 \cdot s_{i,k}}+1}
  = a_{i,k}{\rm ReLU}_{{\rm soft}}^{(n)}(s_{i,k}).
\end{align*}
For query $q \ne i$, the scalar output at $q$-th column satisfies:
\begin{align*}
  ({\rm Attn}^{(i,k)}(Z^{(2)}))_{1,q}
  = a_{i,k}s_{i,k}\cdot\frac{n}{n + e^{\lambda_3 C}}
  \leq \frac{nC_s}{e^{\lambda_3 C}}.
\end{align*}

The total suppression error at any non-target column $q \ne i$, summed over all $nN$ heads, is bounded by
\begin{align*}
  \sum_{i=1}^{n}\sum_{k=1}^{N} \frac{nC_s}{e^{\lambda_3 C}} = \frac{n^2 NC_s}{e^{\lambda_3 C}}.
\end{align*}
Setting this less than $ \epsilon_{{\rm relu}}$ gives:
\begin{align*}
  C \geq \frac{\ln(n^2 NC_s/\epsilon_{{\rm relu}})}{\lambda_3}.
\end{align*}
Summing over all $nN$ heads gives our construction of Layer 3
\begin{align*}
  {\rm Attn}_3(Z^{(2)}) := \sum_{i=1}^{n}\sum_{k=1}^{N} {\rm Attn}^{(i,k)}(Z^{(2)}).
\end{align*}
Since head $(i,k)$ writes only to column $i$ of the $1\times(n+1)$ output, contributions from different $i$ do not interfere.
The final output row vector is:
\begin{align*}
  ({\rm Attn}_3(Z^{(2)}))_{1,j} =
  \begin{cases}
    \displaystyle\sum_{k=1}^{N} a_{j,k}{\rm ReLU}_{{\rm soft}}^{(n)}(s_{j,k}) + O(\epsilon_{{\rm relu}}), & j \leq n, \\[6pt]
    0, & j = n+1.
  \end{cases}
\end{align*}
After truncation by $T$, we obtain the $1\times n$ row vector $(\hat{f}_1, \ldots, \hat{f}_n)$ approximating $(f_1, \ldots, f_n)$.

By \cref{lem:softrelu_approximate_relu} and our construction above, the parameter bound of Layer 3 satsfies
\begin{align}\label{eq:parameter_bound_new_3}
    H_3 = nN,\quad
  \lambda_3 = O(\frac{\ln(C_sn/\epsilon_{{\rm relu}})}{\epsilon_{{\rm relu}}}),\quad \|W_V^{(3)}\|_F = 1, \quad \|W_{KQ}^{(3)}\|_F=O(C\sqrt{n}).
\end{align}

\textbf{Error Analysis}

Recall that in output of Layer 1 we have
\begin{align*}
    \|\hat{X}' - X'\|_{\rm max} \leq \epsilon_1.
\end{align*}
Therefore, in Layer 2 we have:
\begin{align*}
  |\hat{s}_{i,k} - s_{i,k}|
  \leq & ~ \|A_{i,k}(\hat{X}'-X')B\|_{\max} + \epsilon_2 \notag\\
  \leq & ~ \|A_{i,k}\|_\infty \cdot \|\hat{X}'-X'\|_{\rm max} \cdot \|B\|_{1} + \epsilon_2 \\
  \leq & ~ dn\epsilon_1+\epsilon_2.
\end{align*}
The actual Layer 3 output at query $i$ is $\sum_{k=1}^N a_{i,k}{\rm ReLU}_{{\rm soft}}^{(n)}(\hat{s}_{i,k})$, and the target is $f_i = \sum_{k=1}^N a_{i,k}{\rm ReLU}(s_{i,k})$.
For each fixed $(i,k)$, we insert ${\rm ReLU}(\hat{s}_{i,k})$ as an intermediate term:
\begin{align*}
  & ~ |a_{i,k}{\rm ReLU}_{{\rm soft}}^{(n)}(\hat{s}_{i,k}) - a_{i,k}{\rm ReLU}(s_{i,k})|
  \\
  \leq & ~
  |{\rm ReLU}_{{\rm soft}}^{(n)}(\hat{s}_{i,k}) - {\rm ReLU}(\hat{s}_{i,k})|
  +
  |{\rm ReLU}(\hat{s}_{i,k}) - {\rm ReLU}(s_{i,k})| \\
  \leq & ~
  2\epsilon_{\rm relu}+dn\epsilon_1+\epsilon_2.
\end{align*}
The overall error is
\begin{align*}
  |\hat{f}_i - f_i|
  \leq & ~\sum_{k=1}^{N}|a_{i,k}\operatorname{ReLU}_{{\rm soft}}^{(n)}(\hat{s}_{i,k}) - a_{i,k}\operatorname{ReLU}(s_{i,k})| + {\rm (suppression\ error)} \notag\\
  \leq & ~\sum_{k=1}^{N}(2\epsilon_{{\rm relu}} + dn\epsilon_1 + \epsilon_2) + \epsilon_{{\rm relu}} \notag\\
  = & ~ (2N+1)\epsilon_{{\rm relu}} + Ndn\epsilon_1 + N\epsilon_2.
\end{align*}
Therefore, to achieve precision of $\epsilon$, it suffices to set
\begin{align*}
    \epsilon_1=\frac{\epsilon}{3dnN},\quad \epsilon_2=\frac{\epsilon}{3N}, \quad \epsilon_{\rm relu}=\frac{\epsilon}{6N+3}. 
\end{align*}
Finally, by \eqref{eq:parameter_bound_new_1}, \eqref{eq:parameter_bound_new_2} and \eqref{eq:parameter_bound_new_3}, we get the universal parameter bound of the three-layer attention:
\begin{align*}
  & ~H = O(n^2N),\quad W=O(dnN), \quad
  \lambda = O(\frac{N\ln(dnNC_Xw_0/\epsilon)}{\epsilon}),\\ 
  & ~ \|W_V\|_F = w_0\sqrt{dN}+n\sqrt{nd}, \quad \|W_{KQ}\|_F=O(n\sqrt{n\ln\frac{nN C_X w_0}{\epsilon}}).
\end{align*}
This completes the proof.
\end{proof}

\begin{remark}
By letting $w_{i,k,j}':=(w_{i,k,j},b_{i,k})\in\R^{d+1}$ and $x'=(x,1/n)\in\R^{d+1}$, \cref{lem:approximate_one_layer_relu_with_attention_new} also applies to ReLU network with bias in the form 
\begin{align*}
    f_i=
    \sum_{k=1}^N a_{i,k}{\rm ReLU}(\sum_{j=1}^nw_{i,k,j}^\top x_j+b_{i,k}),~ a_{i,k}\in\{-1,1\},~w_{i,k,j}\in\R^d,
    f =
    (f_1,\dots,f_n).
\end{align*}
\end{remark}

\subsection{Main Proof of \texorpdfstring{\cref{lem:maintex_approximate_one_layer_relu_with_attention_matrix}}{}}
\label{subsec:main_proof_one_layer_relu}

In this section, we present the main proof of \cref{lem:maintex_approximate_one_layer_relu_with_attention_matrix}.
We extend \cref{lem:approximate_one_layer_relu_with_attention_new} from the vector-output case to the matrix-output case required in the main text. 
The key idea is that for each output row, we use the construction of \cref{lem:approximate_one_layer_relu_with_attention_new} and then stack these row-wise approximators into a single attention network by modifying the weight matrices.

\begin{lemma}[\cref{lem:maintex_approximate_one_layer_relu_with_attention_matrix} Restated]\label{lem:approximate_one_layer_relu_with_attention_matrix_new}
Let $\mathcal{X}\subset\R^{d\times n}$ be compact. For $X=(x_1,\dots,x_n)\in\mathcal{X}$, let $f:\mathcal{X}\to\R^{d\times n}$ be defined entry-wise by
\begin{align*}
[f(X)]_{r,i}
=
\sum_{k=1}^N a_{i,k}^{(r)}
\relu (\sum_{j=1}^n (w_{i,k,j}^{(r)})^\top x_j),
\qquad r\in[d],\ i\in[n],
\end{align*}
where $a_{i,k}^{(r)}\in\{-1,1\}$ and $w_{i,k,j}^{(r)}\in\R^d$.
Then, for any $0<\epsilon<1$, there exists a Transformer
\begin{align*}
\phi = T\circ \mathrm{Attn}^{(3)}\circ \mathrm{Attn}^{(2)}\circ \mathrm{Attn}^{(1)}\circ P,
\end{align*}
such that
\begin{align*}
\|\phi(X)-f(X)\|_{\max}\le\epsilon
\qquad\text{for all } X\in\mathcal{X},
\end{align*}
where $P$ is the preprocessing map that pads one zero token, and $T$ is the truncation map that deletes the last token.
Further, suppose $\calX\in[-C_X, C_X]^{d\times n}$ and $\|w_{i,k,j}^{(r)}\|_\infty\leq w_0$ for all $i,j \in [n], k \in [N], r\in [d]$.
For $C_X\geq1$, the parameter bounds of constructed Transformer network $\phi$ satisfy
\begin{align*}
  & ~H = O(N),\quad W=O(N), \quad
  \lambda = O(\frac{N\ln(NC_Xw_0/\epsilon)}{\epsilon}),\\ 
  & ~ \|W_V\|_F = O(w_0\sqrt{N}), \quad \|W_{KQ}\|_F=O(\sqrt{\ln\frac{N C_X w_0}{\epsilon}}),
\end{align*}
where $O(\cdot)$ hides polynomial factors depending on $d$ and $n$.
\end{lemma}

\begin{proof}
Define that $f^{(r)}:=(f_1^{(r)},\cdots f_n^{(r)})$.
For $\epsilon>0$, by \cref{lem:approximate_one_layer_relu_with_attention_new} there exists a network with three multi-head attention layers $\phi^{(r)}$, satisfying
\begin{align*}
    \|\phi^{(r)}(X)-f^{(r)}(X)\|_{\rm max}\leq \epsilon.
\end{align*}
Our goal is to stack the networks to obtain $\phi=
\begin{bmatrix}
    \phi^{(1)}\\
    \vdots \\
    \phi^{(d)}
\end{bmatrix}$.
We construct $\phi$ by modifying weight matrices. 

We first recall the intermediate dimensions from the construction of \cref{lem:approximate_one_layer_relu_with_attention_new}.
Following the proof of \cref{lem:approximate_one_layer_relu_with_attention_new}, the three attention layers have the following per-row input/output dimensions:
\begin{itemize}
\item Layer 1 (Zooming Layer): input dimension $d_0 := d + n + 1$, output dimension $d_1 := dnN + n + 1$, where the zooming outputs $Y_{i,k}(X)$ are stacked for all $(i,k)$ pairs together with $I_{n+1}$.
\item Layer 2 (Accumulation Layer): input dimension $d_1 = dnN + n + 1$, output dimension $d_2 := nN + n + 1$, where the accumulated inner products $s_{i,k}$ are stacked together with $I_{n+1}$.
\item Layer 3 (Soft-ReLU Layer): input dimension $d_2 = nN + n + 1$, output dimension $d_3 := 1$, producing a $1 \times (n+1)$ row vector approximating $(f_{r,1}, \ldots, f_{r,n}, 0)$.
\end{itemize}
We then construct selector matrices and organizer matrices accordingly to stack output to obtain a matrix output.

\textbf{Selector Matrices.}

For each layer $m \in \{1, 2, 3\}$ and each $r \in [d]$, the selector matrix $Q_r^{(m)}$ extracts the $r$-th block from the stacked representation at the input of layer $m$.
Since the stacked representation before layer $m$ has row dimension $d \cdot d_{m-1}$, with each block of dimension $d_{m-1}$, we define:
\begin{itemize}
\item For layer 1, the input is $P(X) \in \R^{ d_0 \times (n+1)}$ with $d_0 = d + n + 1$.
We do not need a selector matrix for layer 1 since we do not stack the preprocessing layer.
To keep the form align with other layers, we define
\begin{align*}
Q_r^{(1)} := I_{d_0}, \quad r \in [d].
\end{align*}

\item For layer 2, the input is the stacked output of layer 1 in $\R^{d \cdot d_1 \times (n+1)}$ with $d_1 = dnN + n + 1$.
We define
\begin{align*}
Q_r^{(2)} := \begin{bmatrix} 0_{(r-1)d_1 \times d_1} \\ I_{d_1} \\ 0_{(d-r)d_1 \times d_1} \end{bmatrix}^\top \in \R^{d_1 \times d \cdot d_1}, \quad r \in [d].
\end{align*}
Then $Q_r^{(2)}$ extracts the $r$-th block of the stacked zooming-layer output.

\item For layer 3, the input is the stacked output of layer 2 in $\R^{d \cdot d_2 \times (n+1)}$ with $d_2 = nN + n + 1$.
We define
\begin{align*}
Q_r^{(3)} := \begin{bmatrix} 0_{(r-1)d_2 \times d_2} \\ I_{d_2} \\ 0_{(d-r)d_2 \times d_2} \end{bmatrix}^\top \in \R^{d_2 \times d \cdot d_2}, \quad r \in [d].
\end{align*}
Then $Q_r^{(3)}$ extracts the $r$-th block of the stacked accumulation-layer output.

\end{itemize}

\textbf{Organizer matrices.}
For each layer $m \in \{1, 2, 3\}$ and each $r \in [d]$, the organizer matrix $O_r^{(m)}$ embeds the per-row output of dimension $d_m$ into the $r$-th block of the stacked representation of dimension $d \cdot d_m$.
We define
\begin{align*}
O_r^{(1)} := & ~ 
\begin{bmatrix} 0_{(r-1)d_1 \times d_1} \\ I_{d_1} \\ 0_{(d-r)d_1 \times d_1} \end{bmatrix} \in \R^{d \cdot d_1 \times d_1}, \quad r \in [d],\\
O_r^{(2)} := & ~
\begin{bmatrix} 0_{(r-1)d_2 \times d_2} \\ I_{d_2} \\ 0_{(d-r)d_2 \times d_2} \end{bmatrix} \in \R^{d \cdot d_2 \times d_2}, \quad r \in [d],\\
O_r^{(3)} := & ~
\begin{bmatrix} 0_{(r-1)d_3 \times d_3} \\ I_{d_3} \\ 0_{(d-r)d_3 \times d_3} \end{bmatrix} \in \R^{d \cdot d_3 \times d_3}, \quad r \in [d].
\end{align*}
For $r \neq r'$, the matrices $O_r^{(m)}$ and $O_{r'}^{(m)}$ have disjoint nonzero row supports.
Left-multiplying by $O_r^{(m)}$ embeds a $d_m$-dimensional vector into the $r$-th block of the $d \cdot d_m$-dimensional stacked vector, with zeros elsewhere.

\textbf{Final Construction.}

For each layer $m \in \{1, 2, 3\}$, suppose the per-row network $\phi^{(r)}$ has $H_r^{(m)}$ attention heads in its $m$-th layer.
Using the selector matrices $Q_r^{(m)} \in \R^{d_{m-1} \times d \cdot d_{m-1}}$ and organizer matrices $O_r^{(m)} \in \R^{d \cdot d_m \times d_m}$ defined in Step~3, we construct the $m$-th global attention layer as
\begin{align*}
{\rm Attn}^{(m)}(Z) = \sum_{r=1}^{d} \sum_{h=1}^{H_r^{(m)}} O_r^{(m)} W_V^{(h,r,m)} Q_r^{(m)} Z ~ {\rm Softmax}_\lambda( (W_K^{(h,r,m)} Q_r^{(m)} Z)^\top W_Q^{(h,r,m)} Q_r^{(m)} Z ).
\end{align*}
To see that this is indeed a valid multi-head attention layer, we define the modified weight matrices for the global network as
\begin{align*}
\overline{W}_V^{(h,r,m)} := O_r^{(m)} W_V^{(h,r,m)} Q_r^{(m)}, \quad \overline{W}_K^{(h,r,m)} := W_K^{(h,r,m)} Q_r^{(m)}, \quad \overline{W}_Q^{(h,r,m)} := W_Q^{(h,r,m)} Q_r^{(m)}.
\end{align*}
Then the global layer takes the standard multi-head attention form
\begin{align*}
{\rm Attn}_m(Z) = \sum_{r=1}^{d} \sum_{h=1}^{H_r^{(m)}} \overline{W}_V^{(h,r,m)} Z ~ {\rm Softmax}_\lambda( (\overline{W}_K^{(h,r,m)} Z)^\top \overline{W}_Q^{(h,r,m)} Z ),
\end{align*}
with total number of heads $H^{(m)} = \sum_{r=1}^d H_r^{(m)}$ in layer $m$.

Finally, by \cref{lem:approximate_one_layer_relu_with_attention_new}, we have the parameter bound of the constructed attention network:
\begin{align*}
& ~H = O(dn^2N),\quad W=O(d^2nN), \quad
  \lambda = O(\frac{N\ln(dnNC_Xw_0/\epsilon)}{\epsilon}),\\ 
  & ~ C_V = O(w_0\sqrt{dN}+n\sqrt{nd}), \quad C_{KQ}=O(n\sqrt{n\ln\frac{nN C_X w_0}{\epsilon}}).
\end{align*}
Dropping polynomial factors depending on $d$ and $n$ completes the proof.
\end{proof}

\begin{remark}\label{rem:has_relu_form}
The intermediate ReLU functions $f$ that appears in \cref{lem:approximate_one_layer_relu_with_attention_new} is consistent with \cref{def:ffn_class} in \cref{sec:preliminary}.
Specifically, $f_r:\R^{d\times n}\rightarrow\R^n$ belongs to $\calF(2, W, S, B)$ on the vectorized input $x={\rm vec}(X)\in\R^{dn}$
\begin{align*}
f_{r}(x_1,\dots,x_n):= (A_2~\mathrm{ReLU}[~\cdot~])\circ(A_1x),
\end{align*}
with weight matrix $A_1\in\R^{nN\times dn}$ in first layer formed by
\begin{align*}
(A_1)_{(i-1)N+k,:}= ( (w_{i,k,1}^{(r)})^\top |~  (w_{i,k,2}^{(r)})^\top  |~ \cdots \ |~  (w_{i,k,n}^{(r)})^\top)\in\R^{1\times dn},
\end{align*}
and weight matrix $A_2\in\R^{n\times nN}$ in second layer with the form
\begin{align*}
A_2=
\begin{bmatrix}
a_{1,1}^{(r)} & \cdots & a_{1,N}^{(r)} & 0 & \cdots & 0 & \cdots & 0\\
0 & \cdots & 0 & a_{2,1}^{(r)} & \cdots & a_{2,N}^{(r)} & \cdots & 0\\
\vdots & & \vdots & \vdots & & \vdots & \ddots & \vdots\\
0 & \cdots & 0 & 0 & \cdots & 0 & \cdots & a_{n,N}^{(r)}
\end{bmatrix}.
\end{align*}
While the origin output of $f$ is $(f_1^\top,\cdots,f_n^\top)^\top\in\R^{dn}$, in this lemma we rearrange the ReLU output to $(f_1,\cdots, f_n)\in\R^{d\times n}$ for convenience of comparison with output of $\phi$.
\end{remark}
\begin{remark}\label{rem:removed_affine_layer}
In the process of stacking ReLU layers, we are able to remove the pre-processing and truncation layers present in intermediate layers of the constructed attention network through modifications to the attention layers.
Consequently, we only keep a single pre-processing layer at the beginning and a single truncation layer at the end of the network.
Specifically, the approximation of a composite ReLU function ${\rm ReLU}_m\circ\cdots\circ{\rm ReLU}_1$ simplifies to the form $T\circ {\rm Attn}_{3m}\circ\cdots\circ{\rm Attn}_1\circ P$, rather than $T\circ {\rm Attn}_{3m}\circ\cdots\circ{\rm Attn}_{3m-2}\circ P\circ \cdots \circ T\circ {\rm Attn}_3\circ\cdots\circ{\rm Attn}_1\circ P$.

Concretely, we achieve this by modifying the weight matrices of ${\rm Attn}_{3i}$, $1\leq i \leq m-1$.
Notice that output of layer ${\rm Attn}_{3i-1}$, $1\leq i \leq m-1$ has positional encoding in last $n+1$ rows.
Therefore, same as the proof of \cref{lem:attention_realize_entrywise_multiplication}, adding one additional identity-preserving head and adding 
organizer matrices in other head maintain $I_{n+1}$ in last $n+1$ rows.
We denote the output of unmodified ${\rm Attn}_{3i}$ as
\begin{align*}
    \begin{bmatrix}
        f_1^{({\rm relu}_i)} & \dots & f_n^{({\rm relu}_i)} & 0_d 
    \end{bmatrix}\in\R^{d\times (n+1)},
\end{align*}
The output of modified ${\rm Attn}_{3i}$ then approximates %
\begin{align*}
    \begin{bmatrix}
        f_1^{({\rm relu}_i)} & \dots & f_n^{({\rm relu}_i)} & 0_d \\
        &I_{n+1}
    \end{bmatrix}\in\R^{(d+n+1)\times (n+1)},
\end{align*}
which is exactly in the form of output of preprocessing layer.
Thus, we are able to remove the truncation layer $T$ and preprocessing layer $P$ between ${\rm Attn}_{3i}$ and ${\rm Attn}_{3i+1}$ for $1\leq i \leq m-1$.
\end{remark}

\subsection{Main Proof of \texorpdfstring{\cref{thm:maintex_approx_relu_trans}}{}}
\label{subsec:main_proof_multi_layer_relu}
In this section, we present the main proof of \cref{thm:maintex_approx_relu_trans}, extending the translation theorem to the ReLU network with general depth.
For each layer of the ReLU network, we construct a transformer block that approximates the corresponding layer on the relevant compact domain. 
We then control the approximation error under composition via the Lipschitz continuity of the ReLU network.

\begin{theorem}[\cref{thm:maintex_approx_relu_trans} Restated]
Let $X:=(x_1,\dots,x_n)\in [-C_X, C_X]^{d \times n}$ and  ${\rm vec}(X):=(x_1^\top,\dots,x_n^\top)\in[-C_X, C_X]^{dn}$.
For any $\epsilon \in (0,1)$ and any ReLU network
$f \in \mathcal{F}(K_f, W_f, S, B)$, there exists  a Transformer
$g \in \mathcal{T}(H, W, K, C_{KQ}, C_V)$
such that
\begin{align*}
\| g(X) - f({\rm vec}(X)) \|_{\max} \leq \epsilon
\end{align*}
for all $X \in [-C_X, C_X]^{d \times n}$,
where $g$ takes the form
$g = T \circ {\rm Attn}_{3K_f} \circ \cdots \circ {\rm Attn}_1 \circ P$
with $P$ the preprocessing layer and $T$ the truncation layer.
For $C_X\geq1$, the parameters of $g$ satisfy
\begin{align*}
    & ~H = O( W_f),
\quad
W = O(W_f),
\quad
K = O(K_f),
\quad \quad C_V = O(B\sqrt{W_f}), \\
& ~ C_{KQ}
= O(\sqrt{K_f \ln(\frac{K_f W_f B C_X}{\epsilon})}), \quad
\lambda= O(\frac{K_f^2W_f^{K_f+1} B^{K_f} \ln(K_f W_f B C_X / \epsilon)}{\epsilon}).
\end{align*}
\end{theorem}

\begin{proof}
We write $f = f_{K_f} \circ \cdots \circ f_{1}$,
where $f_{k}(z) = A_k{\rm ReLU}[z]+b_k$
with $A_k \in \mathbb{R}^{d_{k+1} \times d_k}$,
$\|A_k\|_{\max} \leq B$, and $d_k \leq W_f$ for all $k$.
This is consistent with \cref{def:ffn_class}, since affine map is representable by a one-layer ReLU network.
Define the partial compositions
$f^{(k)} := f_{k} \circ \cdots \circ f_{1}$
and
$g^{(k)} := g_{k} \circ \cdots \circ g_{1}$,
where each $g_{k}$ is the three-attention-layer approximator
from \cref{lem:approximate_one_layer_relu_with_attention_matrix_new},
with intermediate preprocessing and truncation layers
eliminated via \cref{rem:removed_affine_layer}.

We start with bounding intermediate input domain.
Since every entry of $A_k$ satisfies $|(A_k)_{ij}| \leq B$
and each row of $A_k$ contains at most $d_k \leq W_f$ entries,
the induced $\ell_\infty$-operator norm satisfies
$\|A_k\|_{\infty,\infty} \leq W_f B$.
Combined with the $1$-Lipschitz property of ${\rm ReLU}$
under $\|\cdot\|_\infty$, this gives
\begin{align*}
\| f_{k}(z) \|_{\rm max} \leq W_f B \| z \|_{\rm max}
.
\end{align*}
By induction on $k$,
$\| f_k({\rm vec}(X)) \|_{\rm max} \leq (W_f B)^k C_X$
for all $X \in [-C_X, C_X]^{d \times n}$ and $k \in [K_f]$.
Denote $C^{(k)} := (W_f B)^k C_X$ for $k = 0, 1, \ldots, K_f$.

Next we consider the pre-layer approximation through \cref{lem:approximate_one_layer_relu_with_attention_matrix_new}.

Fix per-layer tolerances $\epsilon_k > 0$ to be specified.
By \cref{lem:approximate_one_layer_relu_with_attention_matrix_new} applied with input-domain bound $C^{(k-1)}$
and precision $\epsilon_k$,
for each $k \in [K_f]$ there exists a three-attention-layer
network $g_{k}$ such that,
for all $Z$ satisfying $\|Z\|_{\max} \leq C^{(k-1)}$,
\begin{align*}
\| g_{(k)}(Z) - f_{k}({\rm vec}(Z)) \|_{\max} \leq \epsilon_k.
\end{align*}
The associated parameter bounds from \cref{lem:approximate_one_layer_relu_with_attention_matrix_new} are
\begin{align}\label{eq:parameter_bound_general_depth_epsilon_k}
& ~H^{(k)} = O(dn^2 W_f),
\quad
W^{(k)} = O(d^2 n W_f),
\quad
C_V^{(k)} = O(B\sqrt{dW_f} + n\sqrt{nd}),\notag\\
& ~\lambda^{(k)}
= O(\frac{W_f \ln(dn W_f C^{(k-1)} B / \epsilon_k)}{\epsilon_k}),
\quad
C_{KQ}^{(k)}
= O(n \sqrt{n \ln(\frac{n W_f C^{(k-1)} B}{\epsilon_k})}).
\end{align}
Next, we determine the $\epsilon_k$.
We denote that $\eta_k:=\| g^{(k)}(X) - f^{(k)}({\rm vec}(X)) \|_{\max}$.
By the triangle inequality,
\begin{align*}
\| g^{(k)}(X) - f^{(k)}({\rm vec}(X)) \|_{\max}
& ~\leq
\| g_{k}(g^{(k-1)}(X)) - f_{k}({\rm vec}(g^{(k-1)}(X))) \|_{\max} \\
 & ~ \quad + 
\| f_{k}({\rm vec}(g^{(k-1)}(X))) - f_{k}(f^{(k-1)}({\rm vec}(X))) \|_{\max}.
\end{align*}
The first term satisfies
\begin{align*}
\| g_{k}(g^{(k-1)}(X)) - f_k({\rm vec}(g^{(k-1)}(X))) \|_{\max} \leq \epsilon_k.
\end{align*}
For the second term, we use the identity
$\| {\rm vec}(Z_1) - {\rm vec}(Z_2) \|_\infty = \| Z_1 - Z_2 \|_{\max}$
and the $(W_f B)$-Lipschitz property of $f_{k}$
under $\|\cdot\|_\infty$ to obtain
\begin{align*}
\| f_{k}({\rm vec}(g^{(k-1)}(X)))
- f_{k}(f^{(k-1)}({\rm vec}(X))) \|_{\max}
\leq W_f B \cdot \| g^{(k-1)}(X) - f^{(k-1)}({\rm vec}(X)) \|_{\max}.
\end{align*}
Therefore
\begin{align*}
    \eta_k\leq B_fW\cdot\eta_{k-1}+\epsilon_k
\end{align*}
for $k\geq 2$, and $\eta_1=\epsilon_1$.
To ensure the error $\eta_{k}$ less than $\epsilon$ for $k=K_f$, it suffices to set
\begin{align}\label{eq:setting_of_epsilon_k}
    \epsilon_k := \frac{\epsilon}{K_f (W_f B)^{K_f}},
\quad k \in [K_f].
\end{align}
Taking \eqref{eq:setting_of_epsilon_k} into \eqref{eq:parameter_bound_general_depth_epsilon_k}, we obtain:
\begin{align*}
    & ~H = O(dn^2 W_f),
\quad
W = O(d^2 n W_f),
\quad
K = O(K_f),
\quad \\
& ~C_V = O(B\sqrt{dW_f} + n\sqrt{nd}), \quad C_{KQ}
= O(n\sqrt{nK_f \ln(\frac{nK_f W_f B C_X}{\epsilon})}), \\
& ~\lambda
= O(\frac{K_f^2W_f^{K_f+1} B^{K_f} \ln(dnK_f W_f B C_X / \epsilon)}{\epsilon}).
\end{align*}
Dropping polynomial factors depending on $d$ and $n$ completes the proof.
\end{proof}

\clearpage
\section{Constructive Transformer Universal Approximation}
In this appendix we establish \cref{thm:maintex_attention_uap} via a two-step reduction. 
We first invoke a classical ReLU
universal approximation result to obtain a quantitative ReLU pproximator for any target in the Sobolev unit ball. 
We then apply the translation theorem (\cref{thm:maintex_approx_relu_trans}) to lift this ReLU approximator to a softmax-attention network, thereby inheriting both the approximation guarantee and explicit architectural bounds.

\textbf{Organization.}
\cref{subsec:relu_uap} recalls the Sobolev-space ReLU universal
approximation theorem of \cite{yarotsky2017error}, which supplies the
quantitative ReLU constructions used as input to our translation.
\cref{subsec:transformer_uap} combines these bounds with \cref{thm:maintex_approx_relu_trans} to derive the attention-network universal approximation theorem with explicit resource bounds.

\subsection{ReLU Network's Universal Approximation Theorem}\label{subsec:relu_uap}

We begin by recalling the Sobolev-space framework and the classical ReLU approximation result of \cite{yarotsky2017error}, which serves as the starting point for our construction.

\begin{definition}[Sobolev Space]\label{def:sobolev_space}
Let $n$ and $r$ be positive integers.
We define $W^{r,\infty}([0,1]^n)$ to be the set of all functions $f:[0,1]^n \to \R$ whose weak derivatives up to order $r$ belong to $L^\infty([0,1]^n)$.
We define its norm by
\begin{align*}
\|f\|_{W^{r,\infty}([0,1]^n)}
=
\max_{\alpha \in \{0,1,2,\dots\}^n,\ |\alpha| \leq r}
{\rm ess\,sup}_{x \in [0,1]^n}
|D^\alpha f(x)|.
\end{align*}
We further define the unit ball by
\begin{align*}
\mathcal{F}_{n,r}
=
\{
f \in W^{r,\infty}([0,1]^n)
:
\|f\|_{W^{r,\infty}([0,1]^n)} \leq 1
\}.
\end{align*}
\end{definition}

\begin{lemma}[Theorem 1 of \cite{yarotsky2017error}]\label{lem:parameter_bound_of_constructed_relu}
Let $n$ and $r$ be positive integers.
Let $\epsilon \in (0,1)$.
Then for every $f^\star \in \mathcal{F}_{n,r}$ there exists a ReLU network $f_{\rm ReLU} \in \mathcal{F}(K_f,W_f,S_f,B)$ such that
\begin{align*}
\sup_{x \in [0,1]^n}
|f_{\rm ReLU}(x) - f^\star(x)|
\leq
\epsilon.
\end{align*}
Moreover, the parameters satisfy
\begin{align*}
K_f
&=
O(\ln(\epsilon^{-1})),~
W_f=O(\epsilon^{-n/r}\ln(\epsilon^{-1})),
\\
S_f
&=
O(\epsilon^{-n/r}(\ln\epsilon^{-1})),~
B=O(\epsilon^{-1/r}),
\end{align*}
where $O(\cdot)$ hides polynomial factors in $n,r$.
\end{lemma}

\begin{proof}
See proof of \cite[Theorem 1]{yarotsky2017error}.
\end{proof}

\subsection{Main Proof of \texorpdfstring{\cref{thm:maintex_attention_uap}}{}}
\label{subsec:transformer_uap}
In this subsection we prove \cref{thm:maintex_attention_uap} by combining the quantitative ReLU approximation from \cref{subsec:relu_uap} with the translation theorem developed in \cref{sec:translation_theorem}.

\begin{theorem}[\cref{thm:maintex_attention_uap} Restated, Attention network's Universal Approximation Theory]
Let $n$ and $r$ be positive integers.
Suppose $f^\star \in \mathcal{F}_{n,r}$, where $\calF_{n,r}$ is the unit ball of Sobolev space $W^{r,\infty}([0,1]^n)$ following \cref{def:sobolev_space}.
Then for any $\epsilon \in (0,1)$, there exists an attention network
$g \in T(H,W,K,C_{KQ},C_V)$
such that
\begin{align*}
|g(x) - f^\star(x)|
\leq
\epsilon
\end{align*}
for all $x \in [0,1]^n$.
The network parameters satisfy
\begin{align*}
        & ~H = O(\delta^{n/r}\ln\delta),
\quad
W = O(\delta^{n/r}\ln \delta),
\quad
K = O(\ln \delta),
\quad \\
& ~C_V = O(\delta^{(n+2)/2r}\sqrt{\ln\delta}), \quad C_{KQ}
= O(\ln\delta), \\
& ~\lambda
= O(\delta^{\mathrm{polylog}(\delta)}\cdot\ln^{\mathrm{polylog}(\delta)}\delta ),
\end{align*}
where $\delta=\epsilon^{-1}$ and $O(\cdot)$ hides polynomial factors in $n,r$.
\end{theorem}

\begin{proof}
Fix $\epsilon \in (0,1)$.
Applying \cref{lem:parameter_bound_of_constructed_relu} with target accuracy $\epsilon/2$, we obtain that there exists a ReLU network
$f_{\rm ReLU}\in\mathcal{F}(K_f,W_f,S_f,B_f)$ such that
\begin{align*}
\sup_{x \in [0,1]^n}
|f_{\rm ReLU}(x) - f^\star(x)|
\leq
\epsilon/2.
\end{align*}
We apply \cref{thm:maintex_approx_relu_trans} with $C_X = 1$ and target tolerance $\epsilon/2$ to the ReLU network $f_{\rm ReLU}$.
Then there exists $g\in \calT(H,W,K,C_{KQ},C_V)$ such that
\begin{align*}
|g(x) - f_{\rm ReLU}(x)|
\leq
\epsilon/2,
\end{align*}
for all $X \in [0,1]^n$.
Hence, by triangle inequality we have
\begin{align*}
|g(x) - f^\star(x)|
&\leq
|g(x) - f_{\rm ReLU}(x)|
+
|f_{\rm ReLU}(x) - f^\star(x)|
\\
&\leq
\epsilon/2 + \epsilon/2
\\
&=
\epsilon.
\end{align*}
Following \cref{thm:maintex_approx_relu_trans}, the parameter bound satisfies
\begin{align*}
        & ~H = O(\delta^{n/r}\ln\delta),
\quad
W = O(\delta^{n/r}\ln \delta),
\quad
K = O(\ln \delta),
\quad \\
& ~C_V = O(\delta^{(n+2)/2r}\sqrt{\ln\delta}), \quad C_{KQ}
= O(\ln\delta), \\
& ~\lambda
= O(\delta^{\mathrm{polylog}(\delta)}\cdot\ln^{\mathrm{polylog}(\delta)}\delta ),
\end{align*}
where $\delta=\epsilon^{-1}$ and $O(\cdot)$ hides polynomial factors in $n,r$.
This completes the proof.
\end{proof}

\begin{remark}\label{rem:extension_of_uap}
The same argument extends to vector-valued targets and general bounded domains.
More precisely, let $C_X > 0$ and let
\begin{align*}
f^\star
=
(f_1^\star,\dots,f_m^\star)
:
[-C_X,C_X]^n
\to
\R^m.
\end{align*}
Assume that after the affine normalization
\begin{align*}
\bar f_j^\star(u)
=
f_j^\star(2 C_X u - C_X {\bf 1}_n),
\quad
u \in [0,1]^n,
\end{align*}
each coordinate function $\bar f_j^\star$ belongs to $\mathcal{F}_{n,r}$.
Then one first approximates each coordinate by \cref{lem:parameter_bound_of_constructed_relu}, stacks the resulting ReLU networks in parallel, and finally applies \cref{thm:maintex_approx_relu_trans}.
This yields an attention network $g$ satisfying
\begin{align*}
\|g(X) - f^\star(X)\|_{\max}
\leq
\epsilon
\end{align*}
for all $X \in [-C_X,C_X]^{1 \times n}$.
More generally, for sequence-to-sequence targets in the Sobolev class $\mathcal{F}_{d\times n,r}$, the same construction applies and all parameter bounds remain valid with the dimension $n$ replaced by the total sequence dimension $dn$.
\end{remark}

\clearpage

\section{Approximation of Rational Functions with Attention Network}

In this section, we leverage translation theorems in \cref{sec:main_theory} to derive target-specific approximation with attention networks. 
In each case, we first recall a quantitative ReLU construction from the literature, then apply \cref{thm:maintex_approx_relu_trans} to translate it into an attention-based approximator with explicit bounds. 

\paragraph{Organization.} \cref{sec:proof_monomial}, \cref{sec:proof_recipro}, and \cref{sec:proof_min_max} incorporate our translation results (\cref{lem:maintex_approximate_one_layer_relu_with_attention_matrix,thm:maintex_approx_relu_trans}) and the ReLU approximation for the minimum/maximum operation, the monomials, and the reciprocal functions, respectively.

\subsection{Proof of \texorpdfstring{\cref{lem:trans_monomial}}{}}
\label{sec:proof_monomial}
We start with recalling the ReLU approximation result for monomials established in \cite{oko2023diffusion}.

\begin{lemma}[ReLU Approximation of Monomials, Lemma F.6 of \cite{oko2023diffusion}]
\label{lem:relu_monomial}
Let $d \geq 2$, $C_X \geq 1$, and $0\leq \epsilon_{{\rm error}} <1$.
For any $\epsilon \in (0,1)$, there exists a ReLU network $f_{{\rm mult}} \in \mathcal{F}(K_f, W_f, S, B)$ such that
\begin{align*}
|f_{{\rm mult}}(x') - \prod_{i=1}^{d} x_i| \leq \epsilon + d~C_X^{d-1}~\epsilon_{{\rm error}}
\end{align*}
for all $x \in [-C_X, C_X]^d$ and $x' \in \mathbb{R}^d$ with $\|x - x'\|_\infty \leq \epsilon_{{\rm error}}$.
The network parameters satisfy
\begin{align*}
K_f = O(\ln d \cdot (\ln\epsilon^{-1} + d\ln C_X)),\quad W_f = 48d,\quad S = O(d\ln\epsilon^{-1} + d\ln C_X),\quad B = C_X^d,
\end{align*}
where $O(\cdot)$ hides polynomial factors in $d$.
\end{lemma}

\begin{proof}
See proof of \cite[Lemma F.6]{oko2023diffusion}.
\end{proof}

We now present the formal proof of \cref{lem:trans_monomial}.

\begin{proposition}
[\cref{lem:trans_monomial} Restated, Approximation of Monomials]
Let $C_X \geq 1$ and $d, n \in \mathbb{N}$ with $dn > 2$. Then, for any $\epsilon \in (0,1)$, there exist $f_{\mathrm{mult}} \in \calT(H,W,K,C_{KQ}, C_V)$ such that
\begin{align*}
\abs{   f_{\mathrm{mult}}(X')  - \prod_{j=1}^n \prod_{i=1}^d X_{ij} } \leq \epsilon + dn C_X^{dn - 1}\epsilon_{ \mathrm{error} }
\end{align*}
for all  $X \in [-C_X, C_X]^{d \times n}$ and $X' \in \R^{d \times n}$ satisfying $\| X - X' \|_\infty \leq \epsilon_{ \mathrm{error} }$.
For $C_X\geq1$, the network parameters satisfy: 
\begin{align*}
        & ~H = O(1),
\quad
W = O(1),
\quad
K = O(\ln \delta),
\quad \\
& ~C_V = O(C_X^{dn}), \quad C_{KQ}
= O(\ln \delta), \\
& ~\lambda
= O(\delta\cdot\ln^3 \delta\cdot(dn)^{\mathrm{polylog}(\delta)} C_X^{\mathrm{polylog}(\delta)} ),
\end{align*}
where $\delta=C_X/\epsilon$, $\mathrm{polylog}(\delta)$ denotes a polynomial in the logarithm of $\delta$ and $O(\cdot)$ hides polynomial factors depending on $d$ and $n$.
\end{proposition}

\begin{proof}
By \cref{lem:relu_monomial} applied with dimension $dn$ and precision $\epsilon/2$, there exists $f_{{\rm mult}} \in \mathcal{F}(K_f, W_f, S, B)$ with
\begin{align*}
K_f = O(dn\ln\frac{C_X}{\epsilon}),\quad W_f = 48dn,\quad B = C_X^{dn},
\end{align*}
such that
\begin{align}
|f_{{\rm mult}}(x') - \prod_{j=1}^n \prod_{i=1}^d X_{ij}| \leq \frac{\epsilon}{2} + dn~C_X^{dn-1}~\epsilon_{{\rm error}}. \label{eq:relu-mult-bound}
\end{align}
Applying \cref{thm:maintex_approx_relu_trans} with input domain $C_X$ and precision $\epsilon/2$, we obtain $g_{{\rm mult}} \in \mathcal{T}(\cdot)$ satisfying
\begin{align*}
\|g_{{\rm mult}}(X') - f_{{\rm mult}}(x')\|_{\max} \leq \frac{\epsilon}{2}.
\end{align*}
The triangle inequality and \eqref{eq:relu-mult-bound} then give
\begin{align*}
|g_{{\rm mult}}(X') - \prod_{j,i} X_{ij}| \leq \frac{\epsilon}{2} + \frac{\epsilon}{2} + dn~C_X^{dn-1}~\epsilon_{{\rm error}} = \epsilon + dn~C_X^{dn-1}~\epsilon_{{\rm error}}.
\end{align*}

Following \cref{thm:maintex_approx_relu_trans}, the parameter bound satisfies
\begin{align*}
        & ~H = O(1),
\quad
W = O(1),
\quad
K = O(\ln \delta),
\quad \\
& ~C_V = O(C_X^{dn}), \quad C_{KQ}
= O(\ln \delta), \\
& ~\lambda
= O(\delta\cdot\ln^3 \delta\cdot(dn)^{\mathrm{polylog}(\delta)} C_X^{\mathrm{polylog}(\delta)} ),
\end{align*}
where $\delta=\frac{C_X}{\epsilon}$, $\mathrm{polylog}(\delta)$ denotes a polynomial in the logarithm of $\delta$ and $O(\cdot)$ hides polynomial factors depending on $d$ and $n$.

This completes the proof.
\end{proof}

\subsection{Proof of \texorpdfstring{\cref{lem:trans_recipro}}{}}
\label{sec:proof_recipro}

We first recall the ReLU approximation result for the reciprocal function established in \cite{oko2023diffusion}.

\begin{lemma}[ReLU Approximation of Reciprocal Function, Lemma F.7 of \cite{oko2023diffusion}]\label{lem:relu_reciprocal}
For any $0 < \epsilon < 1$, there exists $\phi_{{\rm rec}} \in \mathcal{F}(L,W,S,B)$ with
\begin{align*}
L = O(\ln^2 \epsilon^{-1}), \|W\|_\infty = O(\ln^3 \epsilon^{-1}), S = O(\ln^4 \epsilon^{-1}), B = O(\epsilon^{-2})
\end{align*}
such that
\begin{align*}
\left| \phi_{{\rm rec}}(x') - \frac{1}{x} \right| \leq \epsilon + \frac{|x' - x|}{\epsilon^2},
\end{align*}
for all $x \in [\epsilon, \epsilon^{-1}]$ and $x' \in \R$.
\end{lemma}

\begin{proof}
See proof of \cite[Lemma F.7]{oko2023diffusion}.
\end{proof}

We now present the formal proof of \cref{lem:trans_recipro}.

\begin{proposition}
[\cref{lem:trans_recipro} Restated, Approximation of Reciprocal Function] 
For any approximation error $ \epsilon_{} \in (0,1)$, there exist network $ f_{ \mathrm{inv} } \in \calT(H, W, K, C_{KQ}, C_V) $ such that
\begin{align*}
\abs{ f_{ \mathrm{inv} }(x') - 1/x  } \leq 2\epsilon + \frac{ \abs{x' - x} }{ \epsilon_{}^2 } 
\end{align*}
for all $x \in [\epsilon, 1/\epsilon]$ and $x' \in [-C_X, C_X]$. 
For $C_X\geq1$, the network parameters satisfy:
\begin{align*}
        & ~H = O(\ln^3\delta),
\quad
W = O(\ln^3 \delta),
\quad
K = O(\ln^2 \delta),
\quad \\
& ~C_V = O(\delta^2\sqrt{\ln^3\delta}), \quad C_{KQ}
= O(\sqrt{\ln^2\delta\cdot\ln (\delta C_X)}), \\
& ~\lambda
= O(\delta^{\mathrm{polylog}(\delta)}\cdot\ln^{\mathrm{polylog}(\delta)}\delta\cdot\ln  C_X ),
\end{align*}
where $\delta=1/\epsilon$ and $\mathrm{polylog}(\delta)$ denotes a polynomial in the logarithm of $\delta$.
\end{proposition}

\begin{proof}
By \cref{lem:relu_reciprocal} applied with precision $\epsilon$, there exists $g \in \mathcal{F}(K_f, W_f, S, B)$ with
\begin{align*}
K_f = O(\ln^2\delta),\quad W_f = O(\ln^3\delta),\quad B = O(\delta^2),
\end{align*}
such that
\begin{align}\label{eq:relu-inv-bound}
|g(x') - \tfrac{1}{x}| \leq \epsilon + \frac{|x' - x|}{\epsilon^{2}}. 
\end{align}
Applying \cref{thm:maintex_approx_relu_trans} with input domain $C_X$ and precision $\epsilon$, we obtain $f_\mathrm{inv} \in \mathcal{T}(\cdot)$ satisfying
\begin{align*}
|f_\mathrm{inv}(x') - g(x')| \leq \epsilon.
\end{align*}
Combining with \eqref{eq:relu-inv-bound} via the triangle inequality gives
\begin{align*}
|f_\mathrm{inv}(x') - \tfrac{1}{x}| \leq  2\epsilon + \frac{|x' - x|}{\epsilon^{2}}.
\end{align*}

Following \cref{thm:maintex_approx_relu_trans}, the parameter bound satisfies
\begin{align*}
        & ~H = O(\ln^3\delta),
\quad
W = O(\ln^3 \delta),
\quad
K = O(\ln^2 \delta),
\quad \\
& ~C_V = O(\delta^2\sqrt{\ln^3\delta}), \quad C_{KQ}
= O(\sqrt{\ln^2\delta\cdot\ln (\delta C_X)}), \\
& ~\lambda
= O(\delta^{\mathrm{polylog}(\delta)}\cdot\ln^{\mathrm{polylog}(\delta)}\delta\cdot\ln C_X ),
\end{align*}
where $\delta=1/\epsilon$.

This completes the proof.
\end{proof}

\subsection{Proof of \texorpdfstring{\cref{lem:min_max_approx}}{}}
\label{sec:proof_min_max}

We first establish that we're capable of implementing entry-wise minimum, maximum and clipping operations through a single-layer ReLU network.

\begin{lemma}[ReLU Realization of Entry-wise Minimum, Maximum and Clipping Operations]
\label{lem:relu_maxmin}
For any $d,n\in\mathbb{N}$, we're capable of implementing the entry-wise maximum and minimum of two
matrices $X,Y\in\mathbb{R}^{d\times n}$ by
one layer ReLU networks
$g_{\max}, g_{\min}\in\mathcal{F}(K_f,W_f,S,B)$ with
\begin{align*}
    K_f = 1,\quad W_f = O(dn),\quad S = O(dn),\quad B = O(1).
\end{align*}

Moreover, for any $c \in (0,C_X)$, we're capable of implementing the entry-wise clipping map ${\rm clip}_c : \R^{d \times n} \to \R^{d \times n}$ by a one-layer ReLU network $g_{{\rm clip},c} \in \mathcal{F}(K_f,W_f,S,B)$ with
\begin{align*}
K_f = 1,\quad W_f = O(dn),\quad S = O(dn),\quad B = O(C_X).
\end{align*}
\end{lemma}

\begin{proof}
Noticing that for $a,b,t\in\R,c>0$, it holds
\begin{align*}
    \max\{a,b\} = & ~ {\rm ReLU}(a-b) + {\rm ReLU}(b) - {\rm ReLU}(-b), \\
    \min\{a,b\} = & ~ {\rm ReLU}(a) - {\rm ReLU}(-a) - {\rm ReLU}(a-b), \\
    {\rm clip}_c(t) = & ~ {\rm ReLU}(t+c) - {\rm ReLU}(t-c) - c.
\end{align*}
Applying these identities entry-wise and concatenating inputs yields the following constructions.

\textbf{Entry-wise maximum.}
With input $z := ({\rm vec}(X)^\top,{\rm vec}(Y)^\top)^\top \in \R^{2dn}$, we define
\begin{align*}
g_{\max}(z) = A_2^{\max} {\rm ReLU}(A_1^{\max} z + b_1^{\max}) + b_2^{\max},
\end{align*}
where
\begin{align*}
A_1^{\max} = & ~
\begin{bmatrix}
I_{dn} & -I_{dn}\\
0 & I_{dn}\\
0 & -I_{dn}
\end{bmatrix}
\in \R^{3dn \times 2dn},
\quad
b_1^{\max} = 0 \in \R^{3dn},\\
A_2^{\max} = & ~
\begin{bmatrix}
I_{dn} & I_{dn} & -I_{dn}
\end{bmatrix}
\in \R^{dn \times 3dn},
\quad
b_2^{\max} = 0 \in \R^{dn}.
\end{align*}

\textbf{Entry-wise minimum.}
With input $z := ({\rm vec}(X)^\top,{\rm vec}(Y)^\top)^\top \in \R^{2dn}$, we define
\begin{align*}
g_{\min}(z) = A_2^{\min} {\rm ReLU}(A_1^{\min} z + b_1^{\min}) + b_2^{\min},
\end{align*}
where
\begin{align*}
A_1^{\min} = & ~
\begin{bmatrix}
I_{dn} & 0\\
-I_{dn} & 0\\
I_{dn} & -I_{dn}
\end{bmatrix}
\in \R^{3dn \times 2dn},
\quad
b_1^{\min} = 0 \in \R^{3dn},\\
A_2^{\min} = & ~
\begin{bmatrix}
I_{dn} & -I_{dn} & -I_{dn}
\end{bmatrix}
\in \R^{dn \times 3dn},
\quad
b_2^{\min} = 0 \in \R^{dn}.
\end{align*}

\textbf{Clipping Function.}
With input ${\rm vec}(X)^\top\in \R^{dn}$, we define
\begin{align*}
g_{\rm clip}(z) = A_2^{\rm clip} {\rm ReLU}(A_1^{\rm clip} z + b_1^{\rm clip}) + b_2^{\rm clip},
\end{align*}
where
\begin{align*}
A_1^{\rm clip} = & ~
\begin{bmatrix}
I_{dn} \\
I_{dn} 
\end{bmatrix}
\in \R^{2dn \times dn},
\quad
b_1^{\rm clip} = \begin{bmatrix}
    c\one_{dn} \\
    -c\one_{dn}
\end{bmatrix} \in \R^{2dn},\\
A_2^{\rm clip} = & ~
\begin{bmatrix}
I_{dn} & -I_{dn} 
\end{bmatrix}
\in \R^{dn \times 2dn},
\quad
b_2^{\rm clip} = -c\one_{dn} \in \R^{dn}.
\end{align*}

Then $g_{\max},g_{\min},g_{\rm clip}$ implements entry-wise maximum, minimum and clipping operations respectively.
Parameter bounds follow directly by computation.
\end{proof}

Then we present the proof of \cref{lem:min_max_approx}.

\begin{proposition}
[\cref{lem:min_max_approx} Restated, Element-wise Minimum, Maximum and Clipping]
Let $0<c<C_X$.
For any $\epsilon\in(0,1)$, there exist attention networks
\begin{align*}
    f_{\max}, f_{\min},f_{\rm clip}\in\mathcal{T}(H,W,K,C_{KQ},C_V)
\end{align*}
such that, for all $X,Y\in[-C_X,C_X]^{d\times n}$,
\begin{align*}
    \|f_{\max}(X,Y) - \max(X,Y)\|_{\max} \leq & ~\epsilon, \\
    \|f_{\min}(X,Y) - \min(X,Y)\|_{\max} \leq & ~\epsilon, \\
    \|f_{\rm clip}(X) - {\rm clip}_c(X)\|_{\max} \leq & ~\epsilon,
\end{align*}
where $\max(\cdot,\cdot)$ and $\min(\cdot,\cdot)$ denote entry-wise
operations.

For $C_X\geq1$, the network parameters satisfy
\begin{align*}
        & ~H = O(1),
\quad
W = O(1),
\quad
K = O(1),
\end{align*}
and, for $f_{\rm max}$ and $f_{\min}$:
\begin{align*}
    C_V = O(1), \quad C_{KQ}= O(\sqrt{\ln(\delta C_X)}), \quad \lambda =O(\delta\ln (\delta C_X) ),
\end{align*}
for $f_{\rm clip}$:
\begin{align*}
    C_V = O(C_X), \quad C_{KQ}= O(\sqrt{\ln(\delta C_X)}), \quad \lambda =O(\delta C_X\ln (\delta C_X) ),
\end{align*}
where $\delta=1/\epsilon$ and $O(\cdot)$ hides polynomial factors depending on $d$ and $n$.
\end{proposition}

\begin{proof}
We prove the bound for $f_{\max}$; the arguments for $f_{\min}$ and
$f_{\rm clip}$ are identical in structure, differing only in the weight-norm parameter $B$.

By \cref{lem:relu_maxmin}, there exists
$g_{\max}\in\mathcal{F}(K_f,W_f,S,B)$ with $K_f=1$, $W_f=O(1)$,
$B=O(1)$ that implements the entry-wise maximum of $X$ and
$Y$ on $[-C_X,C_X]^{d\times n}$.

Applying \cref{thm:maintex_approx_relu_trans} to $g_{\max}$ with input-domain bound $C_X$ and
target accuracy $\epsilon$ yields an attention-only network
$f_{\max}\in\mathcal{T}(H,W,K,C_{KQ},C_V)$ satisfying
\begin{align*}
    \|f_{\max}(X,Y) - g_{\max}(z)\|_{\max} \leq \epsilon.
\end{align*}
The triangle inequality then gives
\begin{align*}
    \|f_{\max}(X,Y) - \max(X,Y)\|_{\max} \leq \epsilon.
\end{align*}
Substituting $K_f=1$, $W_f=O(1)$, $B=O(1)$ into the parameter bounds of \cref{thm:maintex_approx_relu_trans} gives
\begin{align*}
        & ~H = O(1),
\quad
W = O(1),
\quad
K = O(1),
\quad \\
& ~C_V = O(1), \quad C_{KQ}
= O(\sqrt{\ln(\delta C_X)}), \\
& ~\lambda
= O(\delta\ln (\delta C_X) ),
\end{align*}
where $\delta=1/\epsilon$ and $O(\cdot)$ hides polynomial factors depending on $d$ and $n$.

This completes the proof.
\end{proof}

\clearpage

\section{Proofs of \texorpdfstring{\cref{sec:outline}}{}}

In this section, we provide the proofs of the target-specific approximation results stated in \cref{sec:outline}. 

\textbf{Organization.}
\cref{sec:proof_sqr_rt} proves the attention-based approximation of the square root. 
In addition, \cref{sec:proof_operator1} and \cref{sec:proof_operator2} prove the approximation results for closed-form time schedules.

\subsection{Proof of \texorpdfstring{\cref{prop:srq_rt}}{}}
\label{sec:proof_sqr_rt}
We first recall the ReLU approximation result for the square root function, obtained by composing the reciprocal and monomial ReLU constructions via Newton iteration.

\begin{lemma}[ReLU Approximation of the Square Root; Lemma F.7 of \cite{fu2024unveil}]\label{lem:relu_sqrt}
For any $\epsilon \in (0,1)$, there exists $g_{\mathrm{root}} \in \calF(K_f, W_f, S, B)$ with
\begin{align*}
K_f = O(\ln^2 \epsilon^{-1}),\quad W_f = O(\ln^3 \epsilon^{-1}),\quad S = O(\ln^4 \epsilon^{-1}),\quad B = O(\epsilon^{-1}),
\end{align*}
such that
\begin{align*}
|g_{\mathrm{root}}(x') - \sqrt{x}~| \leq \epsilon + \frac{|x'-x|}{\sqrt{\epsilon}},
\end{align*}
for all $x \in [\epsilon, \epsilon^{-1}]$ and $x' \in \R$.
\end{lemma}

\begin{proof}
See proof of \cite[Lemma F.7]{fu2024unveil}.
\end{proof}

We now present the formal proof of \cref{prop:srq_rt}.

\begin{proposition}
[Approximation of the Square Root; \cref{prop:srq_rt} Restated]
Let $C_X > 0$.
Then, for any $\epsilon \in (0,1)$, there exists $f_{\mathrm{root}} \in \calT(H, W, K, C_{KQ}, C_V)$ such that
\begin{align*}
|f_{\mathrm{root}}(x') - \sqrt{x}| \leq 2\epsilon + \frac{|x' - x|}{\sqrt{\epsilon}}, \quad \text{for all } x' \in [-C_X, C_X], x \in [\epsilon, 1/\epsilon].
\end{align*}
For $C_X\geq1$, the network parameters satisfy:
\begin{align*}
            & ~H = O(\ln^3\delta),
\quad
W = O(\ln^3 \delta),
\quad
K = O(\ln^2 \delta),
\quad \\
& ~C_V = O(\delta\sqrt{\ln^3\delta}), \quad C_{KQ}
= O(\sqrt{\ln^2\delta\cdot\ln (\delta C_X)}), \\
& ~\lambda
= O(\delta^{\mathrm{polylog}(\delta)}\cdot\ln^{\mathrm{polylog}(\delta)}\delta\cdot\ln  C_X ).
\end{align*}
where $\delta = 1/\epsilon$ and $\mathrm{polylog}(\delta)$ denotes a polynomial in the logarithm of $\delta$.
\end{proposition}

\begin{proof}
By \cref{lem:relu_sqrt} applied with precision $\epsilon$, there exists $g_{\mathrm{root}} \in \calF(K_f, W_f, S, B)$ with
\begin{align*}
K_f = O(\ln^2 \delta),\quad W_f = O(\ln^3 \delta),\quad B = O(\delta),
\end{align*}
such that
\begin{align}
|g_{\mathrm{root}}(x') - \sqrt{x}| \leq \epsilon + \frac{|x'-x|}{\sqrt{\epsilon}}. \label{eq:relu_sqrt_bound}
\end{align}
Applying \cref{thm:maintex_approx_relu_trans} with input domain $C_X$ and precision $\epsilon$, we obtain $f_{\mathrm{root}} \in \calT(\cdot)$ satisfying
\begin{align*}
|f_{\mathrm{root}}(x') - g_{\mathrm{root}}(x')| \leq \epsilon.
\end{align*}
Combining with \eqref{eq:relu_sqrt_bound} via the triangle inequality gives
\begin{align*}
|f_{\mathrm{root}}(x') - \sqrt{x}| \leq |f_{\mathrm{root}}(x') - g_{\mathrm{root}}(x')| + |g_{\mathrm{root}}(x') - \sqrt{x}| \leq 2\epsilon + \frac{|x'-x|}{\sqrt{\epsilon}}.
\end{align*}
Following \cref{thm:maintex_approx_relu_trans}, the parameter bound satisfies
\begin{align*}
            & ~H = O(\ln^3\delta),
\quad
W = O(\ln^3 \delta),
\quad
K = O(\ln^2 \delta),
\quad \\
& ~C_V = O(\delta\sqrt{\ln^3\delta}), \quad C_{KQ}
= O(\sqrt{\ln^2\delta\cdot\ln (\delta C_X)}), \\
& ~\lambda
= O(\delta^{\mathrm{polylog}(\delta)}\cdot\ln^{\mathrm{polylog}(\delta)}\delta\cdot\ln  C_X ).
\end{align*}

This completes the proof.

\end{proof}

\subsection{Proof of \texorpdfstring{\cref{prop:operator1}}{}}
\label{sec:proof_operator1}

We first recall the ReLU approximation result for $\exp(-t/2)$, obtained via Taylor polynomial approximation composed with monomial ReLU constructions.

\begin{lemma}[ReLU Approximation of $\exp(-t/2)$; Lemma  F.8 of \cite{fu2024unveil}]\label{lem:relu_exp}
Let $C_t > 0$.
For any $\epsilon \in (0,1)$, there exists $g_\alpha \in \calF(K_f, W_f, S, B)$ with
\begin{align*}
K_f = O(\ln^2 \epsilon^{-1}),\quad W_f = O(\ln \epsilon^{-1}),\quad S = O(\ln^2 \epsilon^{-1}),\quad B = \exp(O(\ln^2 \epsilon^{-1})),
\end{align*}
such that
\begin{align*}
|g_\alpha(t) - \exp(-t/2)| \leq \epsilon \quad \text{for all } t\geq 0.
\end{align*}
\end{lemma}

\begin{proof}
See proof of \cite{fu2024unveil}[Lemma F.8].
\end{proof}

We now present the formal proof of \cref{prop:operator1}.

\begin{proposition}
[Approximation of $\exp(-t/2)$; \cref{prop:operator1} Restated]
Let $C_X > 0$.
Then, for any $\epsilon \in (0,1)$, there exists $f_\alpha \in \calT(H, W, K, C_{KQ}, C_V)$ such that
\begin{align*}
|f_\alpha(t) - \exp(-t/2)| \leq \epsilon \quad \text{for all } 0 \leq t \leq C_X.
\end{align*}
For $C_X\geq1$, the network parameters satisfy:
\begin{align*}
                & ~H = O(\ln\delta),
\quad
W = O(\ln \delta),
\quad
K = O(\ln^2 \delta),
\quad \\
& ~C_V = \exp(O(\ln^2 \delta))\cdot O(\sqrt{\ln \delta}), \quad C_{KQ}
= O(\sqrt{\ln^2\delta\cdot(\ln^2\delta+\ln C_X)}), \\
& ~\lambda
= \exp(O(\ln^4 \delta))\cdot O(\ln^{\mathrm{polylog}(\delta)}\delta\cdot\ln C_X).
\end{align*}
where $\delta=1/\epsilon$ and $\mathrm{polylog}(\delta)$ denotes a polynomial in the logarithm of $\delta$.
\end{proposition}

\begin{proof}
By \cref{lem:relu_exp} applied with precision $\epsilon/2$, there exists $g_\alpha \in \calF(K_f, W_f, S, B)$ with
\begin{align*}
K_f = O(\ln^2 \delta),\quad W_f = O(\ln \delta),\quad B = \exp(O(\ln^2 \delta)),
\end{align*}
such that
\begin{align}
|g_\alpha(t) - \exp(-t/2)| \leq \frac{\epsilon}{2} \quad \text{for all } 0 \leq t \leq C_X. \label{eq:relu_exp_bound}
\end{align}
Applying \cref{thm:maintex_approx_relu_trans} with input domain $C_X$ and precision $\epsilon/2$, we obtain $f_\alpha \in \calT(\cdot)$ satisfying
\begin{align*}
|f_\alpha(t) - g_\alpha(t)| \leq \frac{\epsilon}{2}.
\end{align*}
Combining with \eqref{eq:relu_exp_bound} via the triangle inequality gives
\begin{align*}
|f_\alpha(t) - \exp(-t/2)| \leq |f_\alpha(t) - g_\alpha(t)| + |g_\alpha(t) - \exp(-t/2)| \leq \epsilon.
\end{align*}

Following \cref{thm:maintex_approx_relu_trans}, the parameter bounds satisfy
\begin{align*}
                & ~H = O(\ln\delta),
\quad
W = O(\ln \delta),
\quad
K = O(\ln^2 \delta),
\quad \\
& ~C_V = \exp(O(\ln^2 \delta))\cdot O(\sqrt{\ln \delta}), \quad C_{KQ}
= O(\sqrt{\ln^2\delta\cdot(\ln^2\delta+\ln C_X)}), \\
& ~\lambda
= \exp(O(\ln^4 \delta))\cdot O(\ln^{\mathrm{polylog}(\delta)}\delta\cdot\ln C_X).
\end{align*}
This completes the proof.
\end{proof}

\subsection{Proof of \texorpdfstring{\cref{prop:operator2}}{}}\label{sec:proof_operator2}

We first recall the ReLU approximation result for $\sqrt{1 - \exp(-t)}$, obtained by composing the ReLU approximations of $\exp(-t)$ and the square root.

\begin{lemma}[ReLU Approximation of $\sqrt{1 - \exp(-t)}$; Lemma F.10 of \cite{fu2024unveil}]\label{lem:relu_sigma}
For any $\epsilon \in (0,1)$ and any $C_X > \epsilon$, there exists $g_\sigma \in \calF(K_f, W_f, S, B)$ with
\begin{align*}
K_f = O(\ln^2 \epsilon^{-1}),\quad W_f = O(\ln^3 \epsilon^{-1}),\quad S = O(\ln^4 \epsilon^{-1}),\quad B = \exp(O(\ln^2 \epsilon^{-1})),
\end{align*}
such that
\begin{align*}
|g_\sigma(t) - \sqrt{1 - \exp(-t)}| \leq \epsilon \quad \text{for all } \epsilon \leq t \leq C_X.
\end{align*}
\end{lemma}

\begin{proof}
See proof of \cite[Lemma F.10]{fu2024unveil}.
\end{proof}

We now present the formal proof of  \cref{prop:operator2}.

\begin{proposition}
[Approximation of $\sqrt{ 1- \exp(-t) }$; \cref{prop:operator2} Restated]
For any $\epsilon \in (0,1)$ and any $C_X > \epsilon$, there exists $f_\sigma \in \calT(H, W, K, C_{KQ}, C_V)$ such that
\begin{align*}
|f_\sigma(t) - \sqrt{1 - \exp(-t)}| \leq \epsilon \quad \text{for all } \epsilon \leq t \leq C_X.
\end{align*}
For $C_X\geq1$, the network parameters satisfy:
\begin{align*}
                & ~H = O(\ln^3\delta),
\quad
W = O(\ln^3 \delta),
\quad
K = O(\ln^2 \delta),
\quad \\
& ~C_V = \exp(O(\ln^2 \delta))\cdot O(\sqrt{\ln^3 \delta}), \quad C_{KQ}
= O(\sqrt{\ln^2\delta\cdot(\ln^2\delta+\ln C_X)}), \\
& ~\lambda
= \exp(O(\ln^4 \delta))\cdot O(\ln^{\mathrm{polylog}(\delta)}\delta\cdot\ln C_X).
\end{align*}
where $\delta=1/\epsilon$, ${\rm polylog}(\delta)$ denotes a polynomial in the logarithm of $\delta$ and $O(\cdot)$ hides polynomial factors depending on $d$ and $n$.
\end{proposition}

\begin{proof}
By \cref{lem:relu_sigma} applied with precision $\epsilon/2$, there exists $g_\sigma \in \calF(K_f, W_f, S, B)$ with
\begin{align*}
K_f = O(\ln^2 \delta),\quad W_f = O(\ln^3 \delta),\quad B = \exp(O(\ln^2 \delta)),
\end{align*}
such that
\begin{align}
|g_\sigma(t) - \sqrt{1 - \exp(-t)}| \leq \frac{\epsilon}{2} \quad \text{for all } \epsilon \leq t \leq C_X. \label{eq:relu_sigma_bound}
\end{align}
Applying \cref{thm:maintex_approx_relu_trans} with input domain $C_X$ and precision $\epsilon/2$, we obtain $f_\sigma \in \calT(\cdot)$ satisfying
\begin{align*}
|f_\sigma(t) - g_\sigma(t)| \leq \frac{\epsilon}{2}.
\end{align*}
Combining with \eqref{eq:relu_sigma_bound} via the triangle inequality gives
\begin{align*}
|f_\sigma(t) - \sqrt{1-\exp(-t)}| \leq |f_\sigma(t) - g_\sigma(t)| + |g_\sigma(t) - \sqrt{1-\exp(-t)}| \leq \epsilon.
\end{align*}
Following \cref{thm:maintex_approx_relu_trans}, the parameter bounds satisfy
\begin{align*}
                & ~H = O(\ln^3\delta),
\quad
W = O(\ln^3 \delta),
\quad
K = O(\ln^2 \delta),
\quad \\
& ~C_V = \exp(O(\ln^2 \delta))\cdot O(\sqrt{\ln^3 \delta}), \quad C_{KQ}
= O(\sqrt{\ln^2\delta\cdot(\ln^2\delta+\ln C_X)}), \\
& ~\lambda
= \exp(O(\ln^4 \delta))\cdot O(\ln^{\mathrm{polylog}(\delta)}\delta\cdot\ln C_X).
\end{align*}
This completes the proof.
\end{proof}

\clearpage

\clearpage

\clearpage
\def\arxivfont{\rm}
\bibliographystyle{plainnat}

\bibliography{refs}

\end{document}